\documentclass[preprint,review,12pt]{elsarticle}
\usepackage{amssymb}
\usepackage{amsmath}
\usepackage{amsthm}

\usepackage{mathtools} % amsmath with fixes and additions
\usepackage{booktabs} % commands to create good-looking tables
\usepackage{tikz} % nice language for creating drawings and diagrams
\usepackage{mathtools}
\usepackage{breqn}
\usepackage{bm}
\usepackage{amsfonts}
\usepackage{units}
\usepackage{makeidx}
\usepackage{graphicx}
\usepackage{cancel}
\usepackage{caption}
\usepackage{subcaption}
\usepackage{indentfirst}
\usepackage{hyperref}
\usepackage{etoolbox}
\usepackage{footnote}
\usepackage{csquotes}
\usepackage{xcolor}
\usepackage{stmaryrd}
\usepackage{makecell}
\usepackage[export]{adjustbox} 
\usepackage{changepage}

\usepackage{etoolbox}  % For patching environments
\AtBeginEnvironment{equation}{\small}
\AtBeginEnvironment{align}{\small}
\AtBeginEnvironment{gather}{\small}
\AtBeginEnvironment{multline}{\small}
\AtBeginEnvironment{equation*}{\small}
\AtBeginEnvironment{align*}{\small}
\AtBeginEnvironment{gather*}{\small}
\AtBeginEnvironment{multline*}{\small}

\makeatletter
\def\largetilde#1{\mathop{\vbox{\m@th\ialign{##\crcr\noalign{\kern3\p@}%
      \sortoftildefill\crcr\noalign{\kern3\p@\nointerlineskip}%
      $\hfil\displaystyle{#1}\hfil$\crcr}}}\limits}
      
\def\sortoftildefill{$\m@th \setbox\z@\hbox{$\braceld$}%
  \braceld\leaders\vrule \@height\ht\z@ \@depth\z@\hfill\braceru$}
\makeatother

\journal{Neurocomputing}

\begin{document}

\begin{frontmatter}


\title{Exact and general decoupled solutions of the Linear Model of Co-regionalization}
\author{O.~Truffinet}
\ead{olivier.truffinet@edf.fr}
\affiliation{organization={EDF R\&D PERICLES},
            addressline={7 boulevard Gaspard Monge},
            city={Palaiseau},
            postcode={91120},
            country={France}}
\author{K.~Ammar}
\affiliation{organization={Université Paris-Saclay, CEA, Service d’Études des Réacteurs et de Mathématiques Appliquées},
            city={Saclay},
            postcode={91190},
            country={France}}
\author{J-P.~Argaud}
\affiliation{organization={EDF R\&D PERICLES},
            addressline={7 boulevard Gaspard Monge},
            city={Palaiseau},
            postcode={91120},
            country={France}}

\author{B.~Bouriquet}
\affiliation{organization={EDF DQI},
            addressline={2 rue Ampère},
            city={Saint-Denis CEDEX},
            postcode={93206},
            country={France}}

%% Abstract
\begin{abstract}
The Linear Model of Co-regionalization (LMC) is a very general multitask gaussian process model for regression or classification. While its expressiveness and conceptual simplicity are appealing, naive implementations have cubic complexity in the product (number of datapoints $\times$ number of tasks), making approximations mandatory for most applications. However, recent work has shown that in some settings the latent processes of the model can be decoupled, leading to a complexity that is only linear in the number of said processes. We here extend these results, showing from the most general assumptions that the only condition necessary to an efficient exact computation of the LMC is a mild hypothesis on the noise model. We introduce a full parametrization of the resulting \emph{projected LMC} model, enabling its efficient optimization. The effectiveness of this approach is assessed through synthetic and real-data experiments, testing in particular the behavior of its underlying noise model restriction.\\
Overall, the projected LMC appears as a competitive and simpler alternative to state-of-the art multitask gaussian process models. It greatly facilitates some computations such as training data updates or leave-one-out cross-validation, and is more interpretable, for it gives access to its low-dimensional quantities and to their explicit relation with the full-dimensional data. These qualities could facilitate the adoption by various industries of entire classes of methodologies, notably multitask bayesian optimization.
\end{abstract}

%% Keywords

\end{frontmatter}

\section{Introduction}\label{intro}
\subsection{Motivation: need for simple and efficient Multi-Output Gaussian Process models}

Multi-Outputs Gaussian Processes (MOGP) are popular tools for learning linearly correlated quantities. They find use in many fields: earth sciences, healthcare, power systems operations... Some recent successful applications include multivariate physiological time-series analysis \cite{med_time_series}, battery health prediction \cite{batteries}, and solar power prediction \cite{solar}. They even provide full paradigms for methodologies such as multitask uncertainty quantification (\cite{uq}) and multitask bayesian optimization (\cite{bo}), further referred to as MTBO. This latter methodology, in particular, is becoming a staple in several engineering fields: MTBO is used for the prediction of properties of chemicals \cite{mtbo_drugs} and alloys \cite{mtbo_alloys}, tuning and calibration of additive manufacturing devices \cite{mtbo_addfab_1} \cite{mtbo_addfab_2} \cite{mtbo_cvd}, or design of complex systems such as analog circuits \cite{mtbo_circuits}, database systems \cite{mtbo_db} and robots \cite{mtbo_robots_2}\cite{mtbo_robots_3}.\\
Despite these successes, MOGPs are still slow to be widely adopted in sectors where they have great potential, for instance surrogate modeling: many practitioners limit themselves to using collections of single-output GPs, even though MOGPs generally yield more accurate results by leveraging transfers of information between tasks -- a statement demonstrated by many of the above-cited works and others (\cite{batteries}, \cite{solar}, \cite{pump}, \cite{AL_addfab}, \cite{structural_health}). There is also a glaring lack of diversity in the tested methods: the vast majority of MOGP applications make use of the same model, the Intrinsic Co-regionalization Model (ICM, presented in more detail in the next section), which expressiveness is limited, as its inter-inputs covariance model is common to all learning tasks. It is likely that more diverse and flexible models could yield improved predictions in many applications.\\

Several reasons can explain this slow and partial adoption of MOGPs in applied communities. Practitioners may struggle to get to grips with these tools: their mathematics are uneasy, and their software implementation significantly more intricate than this of single-output GPs. This is especially the case with variational models \cite{generic_inference}, which often offer the best performance in terms of accuracy and speed, but are much more complex than vanilla GPs. They may also lack various tools which their single-output counterparts possess: advanced likelihood functions, MCMC sampling schemes, approximations, etc. Lastly, computational performance is a key issue: the naive method for computing a MOGP (forming a large covariance matrix connecting all datapoints from all learning tasks) yields complexities in $\mathcal{O}(n^{3}p^{3})$, with $n$ the number of datapoints and $p$ this of tasks, infeasible for large problems. This is the case with standard implementations of the ICM, unless specific and little-known linear algebra tricks are used (see \cite{all_in_the_noise}). An influential survey of research trends in bayesian optimization from 2023 emphasizes this point: "Regarding the surrogate modeling of MTO [Multi-Task Optimization], the commonly used LMC model is criticized for its computational complexity. While some simple models are proposed to alleviate this issue, their prediction qualities can be affected. Hence, the development of effective surrogate models for MTO is a promising direction" \cite{bo_review}. Several of the above-cited works also depict computation costs as a difficulty, and mention the exploration of cheaper models as research directions (\cite{mtbo_db}, \cite{mtbo_addfab_2}, \cite{mtbo_robots_2}, \cite{med_time_series}). For all these reasons, there is a need to develop simpler, more efficient, and expressive MOGP methodologies. A good starting point in this search is the \emph{Linear Model of Co-regionalization} framework, introduced in the next section: the present work aims to investigate this mostly-theoretical object to turn it into a more practical machine learning tool.

\subsection{Existing realizations of the Linear Model of Co-regionalization}
As described in a recent survey of MOGPs (\cite{MOGP_review}), the vast majority of models operating in the most straightforward setting -- supervised learning of several tasks, all akin to each other and treated in the same fashion -- fall into two main categories: this of \emph{convolutional GPs} and this of \emph{linearly correlated GPs}. The latter can essentially be grouped under the term \emph{"Linear model of co-regionalization"} (LMC, \cite{alvarez_review}). Convolutional GPs being a strict generalization of the LMC\footnote{The LMC can indeed be framed as a convolutional GP in which the convolution kernel between any output and latent GP is a Dirac function.}, we focus on the latter in the present work, leaving more involved models for future consideration.\\
The LMC stems from a very natural idea, this of modeling all observed outputs as linear combinations of common unobserved and independent gaussian processes. However, simply injecting this ansatz into the generic MOGP expressions does not improve the generic $\mathcal{O}\left( (np)^{3} \right)$ complexity associated with them (\cite{alvarez_review}), making this naive approach intractable in most cases. Several methods have therefore been developed to circumvent this limitation, which rely on two complementary approaches: either restricting the model to some particular case that is more amenable to computation, or approximating some of its parts. The first category includes the already-introduced \emph{Intrinsic Co-regionalization Model} (ICM, \cite{alvarez_review}), sometimes referred to as Multi-Task Gaussian Process (\cite{MTGP}), in which all latent processes share the same kernel, resulting in a covariance matrix with Kronecker product structure: it is split into an inter-tasks kernel and an inter-inputs kernel. This category also features two models very similar to each other: the Orthogonal Instantaneous Linear Mixing Model (OILMM, \cite{OILMM}) and Generalized Probabilistic Principal Component Analysis (GPPCA, \cite{GPPCA}), in which the mixing matrix encoding correlations between outputs is constrained to be orthogonal.\\
Models of the second category often resort to \emph{variational approaches} (\cite{generic_inference}), which approximate the posterior of the latent processes and optimize a lower bound of the marginal log-likelihood; they include the pioneer Semiparametric Latent Factor Model (\cite{SPLFM}) and the Collaborative multi-output Gaussian Process (\cite{COGP}). We emphasize that all of the cited methods are \emph{subcases or approximations} of the LMC, not distinct models that would extend its generality. On another note, let us mention that massive work has been undertaken to improve the scalability of single-output GPs (see for instance the review of \cite{scalable}), which also benefits to multi-output ones.

\subsection{Contributions}

The present work addresses the computational bottlenecks of the LMC, by questioning its latent structure: \emph{the model being intrinsically low-dimensional, why doesn't this low dimensionality translate into efficient computations}? Doing so, we generalize two recent publications introducing respectively the OILMM and the GPPCA (see above paragraph). These models are subcases of the LMC where assumptions are made on the noise model, and more importantly on the mixing matrix, which is constrained to have orthogonal or orthonormal columns. However, it happens that this assumption of orthogonality is unnecessary: exact and convenient computation can be carried without it, at the expense of a slightly more complex treatment of the noise. The main contributions of this paper are thus as follows:
\begin{itemize}
\item We write general expressions for LMC posteriors and marginal likelihood at the level of latent processes, and show how they can be efficiently computed if a specific condition on the noise model is enforced;
\item We find a parametrization for noise matrices enforcing this condition, and derive from it a computation-efficient expression of the MLL; 
\item We demonstrate the validity of our approach by comparing the resulting \emph{projected LMC} to concurrent approaches, in particular by undertaking a parametric study on synthetic data to assess the effects of the noise hypothesis in various settings;
\item We explain the practical benefits of this new model compared with state-of-the-art ones.
\end{itemize}

\subsection{Background and notations}

A generic LMC model is specified by $\mathbf{y = Hu} + \bm{\epsilon} $, with $\mathbf{H}$ a $p \times q$ mixing matrix without any specific property, $\mathbf{u} = (u_{1},..., u_{q})^{T} \in \mathbb{R}^{q}$ such that $u_{i} \sim \mathcal{GP}(0, k_{i})$ and priors of the $u_{i}$'s are mutually independent, and $\bm{\epsilon} \sim \mathcal{N}(0,\mathbf{\Sigma})$ with $\mathbf{\Sigma} \in \mathbb{R}^{p \times p}$ an inter-task noise matrix which is symmetric but not necessarily diagonal. We also note $\mathbf{X} \in \mathbb{R}^{d \times n}$ the matrix of $d$-dimensional input points,$\mathbf{Y} \in \mathbb{R}^{p \times n}$ this of the $p$ data outputs, $\mathbf{U} \in \mathbb{R}^{q \times n}$ the values of the $q$ latent processes at observation points, $vec(\mathbf{M})$ the vectorization of the matrix $\mathbf{M}$ (i.e its column-wise unfolding), $\mathbf{M_{v}} = vec(\mathbf{M^{T}})$, $\mathbf{x}_{*}$ a test point for prediction, $\hat{\Phi}$ the probabilistic estimator of random variable $\Phi$, $\mathbf{k}_{i*} = \left( k_{i} \left(\mathbf{x}_{*}, \mathbf{x}_{j} \right) \right)_{1 \leq j \leq n}$ and $k_{i**} = k_{i}(\mathbf{x}_{*}, \mathbf{x}_{*})$ the train/test covariance vector and test-test variance for kernel $i$, $Diag(\mathbf{K}_{i})$ the block-diagonal matrix where the $i$-th block is the kernel matrix $\mathbf{K}_{i}$ (the subscript indicating iteration on $i$ is omitted), and $\otimes $ the Kronecker product. For simplicity, we restrict ourselves to the case of zero-mean GPs: the case of a general sum-of-regressors mean function, significantly more convoluted and useful only in specific settings, will be the matter of future work.

The naive exact implementation of the LMC is then written (\cite{alvarez_review}): 
\begin{align}\label{eq:naive_lmc_1}
p(\mathbf{y_{*} | Y}) &= \mathcal{N} \left( \mathbf{k_{*}^{T} \, \mathcal{K}^{-1} \, Y_{v}}, \ \mathbf{k_{**} - k_{*}^{T} \, \mathcal{K}^{-1} k_{*}} \right) \\
-2 \log p(\mathbf{Y}) &= \mathbf{Y_{v}^{T}} \mathcal{K}^{-1} \mathbf{Y_{v}} \, + \, \log |\mathcal{K}| \, + \, \frac{np}{2}\log \pi \label{eq:naive_lmc_2}
\end{align} 
for respectively predictions at a new point $\mathbf{y_{*}}$ and the marginal log-likelihood, with $\mathcal{K} = \mathbf{(H \otimes I_{n}) \, Diag(K_{i}) \, (H^{T} \otimes I_{n}) + \Sigma \otimes I_{n} }
$ the noise-added cross-tasks covariance matrix. This is formally equivalent to the computation of a single-output GP, but with the burden of handling an ($np \times np$) matrix in all operations.

\section{Latent-revealing expressions of LMC estimators}

\subsection{Intuition and pathway}

The following developments build on a simple observation: \emph{the naive implementation of the LMC in eq.\eqref{eq:naive_lmc_1} and \eqref{eq:naive_lmc_2} doesn't make apparent the low-dimensional nature of the model}. All expressions involve the (noise-added) task-level covariance matrix $\mathcal{K}$, which size ($np \times np$) scales with the number of tasks $p$. On another hand, recall that the modeled outputs $\mathbf{y}$ are simply (noise-added) linear combinations of a few latent GPs, here named $\mathbf{u}$: $\mathbf{y = Hu} + \bm{\epsilon} $. If $\hat{\mathbf{y}}_{*}$ is an estimator for $\mathbf{y}$ at test point $\mathbf{x}_{*}$, and $\hat{\mathbf{u}}_{*}$ one for $\mathbf{u}$ likewise, one can easily write the regression estimator in function of the posterior distribution of the latent processes, $p(\hat{\mathbf{u}}_{*}|\mathbf{Y})$ :

\begin{equation}\label{eq:latent_prediction}
    p(\hat{\mathbf{y}}_{*}|\mathbf{Y}) = p(\mathbf{H} \hat{\mathbf{u}}_{*} + \bm{\epsilon}|\mathbf{Y}) = \mathcal{N} \left( \mathbf{H} \, \mathbb{E}(\hat{\mathbf{u}}_{*}|\mathbf{Y}), \ \mathbf{H}\, \mathbb{V}(\hat{\mathbf{u}}_{*}|\mathbf{Y})\, \mathbf{H}^{T} + \mathbf{\Sigma} \right)
\end{equation}

Compared to the previous equation \eqref{eq:naive_lmc_1}, eq.\eqref{eq:latent_prediction} clearly reveals the low-dimensional structure of the model: the calculation boils down to computing the distribution $p(\hat{\mathbf{u}}_{*}|\mathbf{Y})$, which is $q$-dimensional regardless of the number of tasks $p$. However, this lower dimension doesn't guarantee cheaper computations by itself. One case where this computation would obviously be efficient is this of \emph{conditionally independent} latent processes $u_{i}|\mathbf{Y}$, as it would then consist in a batch calculation of $q$ single-output GPs, each with complexity $\mathcal{O}(n^{3})$; it is therefore worthwhile to investigate the conditions leading to such a decoupling.\\

The priors of the LMC latent processes are independent by definition of the model; in these conditions, only an exterior covariance term is susceptible of correlating their posteriors, which only possible source is the observational noise $\bm{\epsilon}$. One should therefore wonder what plays the role of a noise term for these processes; it is not obvious that this term should be independent on any kernel or data values. But we will show that it is in fact the case, and that our intuition stands: latent processes are conditionally independent -- no transfer of information takes place between them -- if and only if their respective "latent noises" are not correlated. This crucial assumption will be framed in definition \ref{def:dpn_cond} as the \emph{diagonally projectable noise} condition.\\

Likewise, it is not obvious that the mean of $p(\hat{\mathbf{u}}_{*}|\mathbf{Y})$ should resemble usual GP estimators, and in particular involve a low-dimensional, purely observational data term (some "projected data"). Let us hint why it is actually so. In order to deduce the latent processes $\mathbf{u}$ from the observed quantities $\mathbf{y}$, we wish to find the inverse operation of lifting data from the latent space to the observation space -- that is, the multiplication by the mixing matrix $\mathbf{H}$. We thus consider a generalized inverse of $\mathbf{H}$, which we will name $\mathbf{T}$: by definition, $\mathbf{TH = I_{q}}$. Let us apply this matrix to the observations: $\mathbf{Ty = THu + T}\bm{\epsilon} = \mathbf{u + T}\bm{\epsilon}$. In this equality, $\mathbf{Ty}$ is itself a gaussian process (as a linear combination of the gaussian processes $\mathbf{y}$) and $\mathbf{T}\bm{\epsilon}$ is a gaussian noise with covariance $\mathbf{T\Sigma T^{T}}$ ($\mathbf{\Sigma}$ being the covariance matrix of $\bm{\epsilon}$). We can therefore hope to identify the gaussian process $\mathbf{u}$ with $\mathbf{Ty}$ (modulo some gaussian noise) for some well-chosen matrix $\mathbf{T}$: in this case, the data seen by the latent processes would be $\mathbf{TY}$. This is exactly what we exhibit next, where we will frame $\mathbf{TY}$ as \emph{projected data} and $\mathbf{\Sigma_{P} = T\Sigma T^{T}}$ as \emph{projected noise}. 

\begin{figure}[hbt]
\centerline{\includegraphics[width=\textwidth]{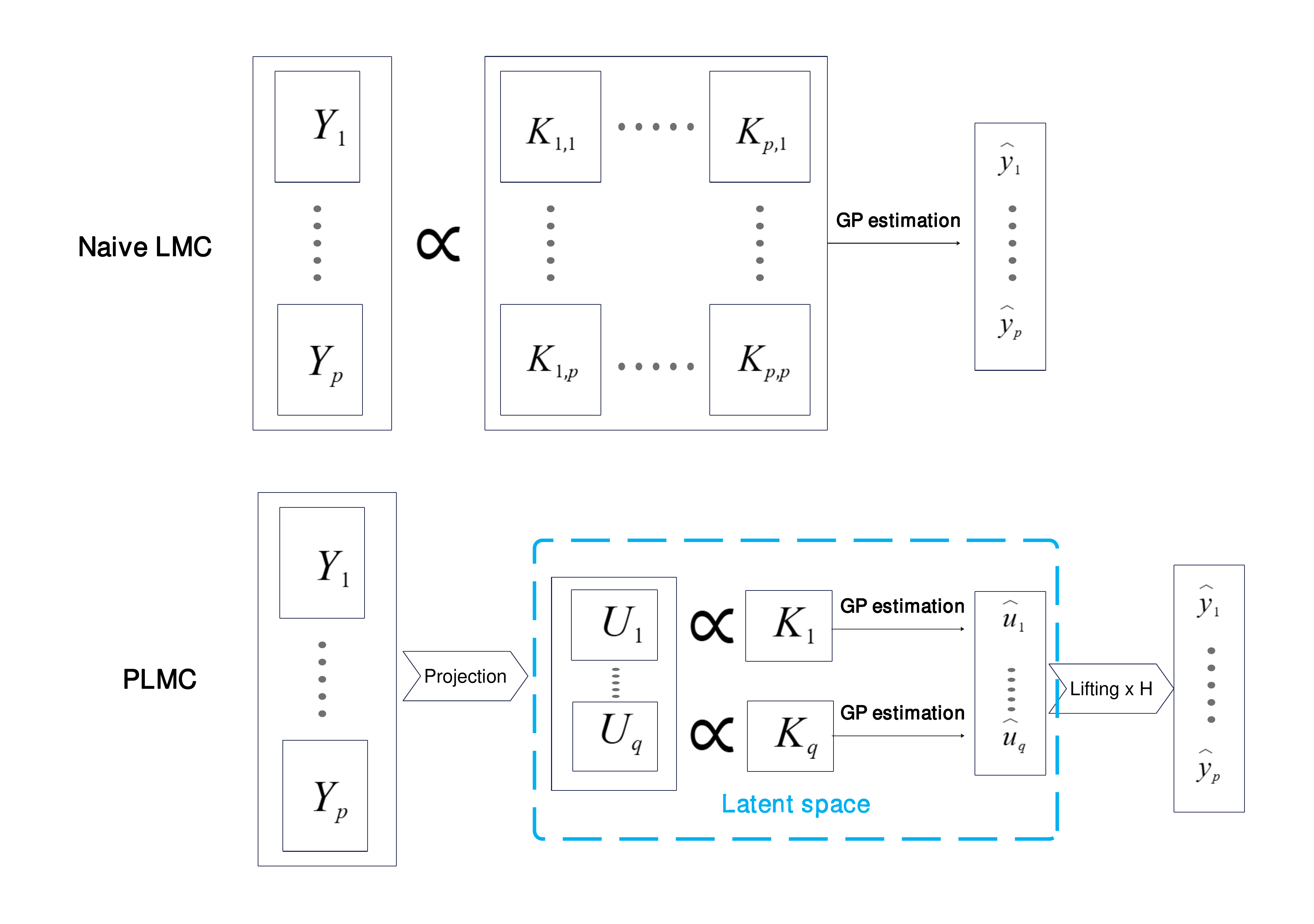}}
\caption{Schematic comparison of the calculation pathways for the naive LMC and the new proposed PLMC. Rather than large $np\times np$ matrix operations, the PLMC projects the observed data into the latent space, performs standard SOGP operations there, then lifts predictions back to the observation space.}
\label{fig:matrix_diag}
\end{figure}

\subsection{Latent-level posteriors and likelihood}\label{sub:posteriors}

\newtheorem{prop}{Proposition}
\newtheorem{rem}{Remark}
\newtheorem{lemma}{Lemma}
\newtheorem{cor}{Corollary}

\begin{prop}\label{prop:estimators}
The posterior distributions of the latent processes, $p(\hat{\mathbf{u}}_{*}|\mathbf{Y})$, is gaussian with respective mean and variance: 
\begin{align}\label{eq:estimators}
\mathbb{E}(\hat{\mathbf{u}}_{*}|\mathbf{Y}) &= \mathfrak{K} \left[ \mathbb{K} + \mathbf{ (H^{T}\Sigma^{-1}H)^{-1} \otimes I_{n}} \right]^{-1} \cdot \mathbf{ vec \left( Y^{T}\Sigma^{-1}H(H^{T}\Sigma^{-1}H)^{-1} \right) } \\
\mathbb{V}(\hat{\mathbf{u}}_{*}|\mathbf{Y}) &= \mathbf{k}_{d} - \mathfrak{K} \left[ \mathbb{K} + \mathbf{ (H^{T}\Sigma^{-1}H)^{-1} \otimes I_{n} } \right]^{-1} \mathfrak{K}
\end{align}
with further notations $\mathbb{K} = \mathbf{Diag(K_{i})}$, $\mathfrak{K} = \mathbf{Diag(k_{i*})}$ and $\mathbf{k}_{d} = \mathbf{Diag}(k_{i**})$.\\
\end{prop}

These expressions are very close to a block-diagonal version of a single-output GP. Two quantities stand out from them: a term $\mathbf{(H^{T}\Sigma^{-1}H)^{-1}}$ which seems to play the role of a noise for the latent processes, and a matrix $\mathbf{ (H^{T}\Sigma^{-1}H)^{-1}H^{T} \Sigma^{-1} Y}$ which appears to act as a $q$-dimensional summary of the data. We therefore set the following notations:

\newtheorem{defi}{Definition}
\begin{defi}\label{def:T_sigma}
We note $\mathbf{\Sigma_{P} = (H^{T}\Sigma^{-1}H)^{-1}}$ and $\mathbf{T = \Sigma_{P}H^{T}\Sigma^{-1}}$. We will later interpret these quantities as a \emph{projected noise} and \emph{projection matrix}.
\end{defi}

\begin{prop}\label{prop:proj_mll}
(taken from \cite{OILMM}) The likelihood of the data can be decomposed as the likelihood of the projected data $\mathbf{TY}$ multiplied by a corrective term accounting for projection loss, independent of the latent processes: 
\begin{equation}\label{eq:proj_mll}
 p(\mathbf{Y}) = \prod_{j=1}^{n} \dfrac{\mathcal{N}(\mathbf{Y_{j}|0, \Sigma})}{\mathcal{N}(\mathbf{TY_{j}|0, \Sigma_{P}})} \times \int p(\mathbf{U}) \prod_{j=1}^{n}\mathcal{N}(\mathbf{TY_{j}|U_{j}, \Sigma_{P}}) \mathbf{dU}
 \end{equation}
(where $\mathbf{Y_{j}, \, U_{j}}$ denote the rows of $\mathbf{Y}$ and $\mathbf{U}$ respectively, associated with datapoint $j$).
\end{prop}

\subsection{Decoupled expressions}\label{sec:decoupled}

Looking at the expressions of section \ref{sub:posteriors}, it appears that they consist in products of block-diagonal matrices with one $nq \times nq$ matrix, $\mathbf{ \mathcal{K}^{-1}} =  \mathbf{\left[ Diag(K_{i}) \ + \  \Sigma_{P} \otimes I_{n} \right]^{-1}}$. This brings about a key observation: \emph{a necessary and sufficient condition for $\mathbf{ \mathcal{K}}$ to be block-diagonal is that $\mathbf{\Sigma_{P}}$ be diagonal}. We frame this condition as of now given its importance in the rest of the article:

\begin{defi}\label{def:dpn_cond}
We say that $\mathbf{\Sigma}$ is a \emph{diagonally projectable noise} for $\mathbf{H}$, abbreviated as DPN, if $\mathbf{H^{T}\Sigma^{-1}H}$ is diagonal.\\
\end{defi}

It follows immediately that if the DPN condition is verified, the estimators of proposition \ref{prop:estimators} decouple: we can express the component $p(\hat{u}_{i*}|\mathbf{Y})$ in function of only the kernel $k_{i}$ and a specific projection of $\mathbf{Y}$ of size $n$.

\begin{prop}\label{prop:decoupled_post}
If the DPN condition stands, we have the following decoupled expression for the posterior of each latent process:
\begin{equation}\label{eq:p_ui}
p(\hat{u}_{i*}|\mathbf{Y}) = \mathcal{N} ( \ \mathbf{k_{i*}^{T} (K_{i}} + \sigma_{i}^{2} \mathbf{I_{n})^{-1} T_{i} Y}, \quad k_{i**} - \mathbf{k_{i*}^{T} (K_{i}} + \sigma_{i}^{2} \mathbf{I_{n})^{-1}k_{i*}} )  
\end{equation}
with $\mathbf{\Sigma_{P} = (H^{T}\Sigma^{-1}H)^{-1}} = Diag(\sigma_{i}^{2})_{1 \leq i \leq q}$ and $\mathbf{T_{i}}$ the $i$-th \emph{row} of $\mathbf{T}$.\\
This is the standard expression of a single-output GP posterior and can thus be computed with time complexity $O(n^{3})$ per latent process.
\end{prop}

\begin{prop}\label{prop:decoupled_lik}
The marginal likelihood factorizes over latent processes if and only if $\mathcal{N}(\mathbf{TY_{j}|U_{j}, \Sigma_{P}})$ does, i.e if $\mathbf{\Sigma_{P}}$ is diagonal. In this case $\mathcal{N}(\mathbf{TY_{j}|U_{j}, \Sigma_{P}}) = \prod_{i=1}^{q} \mathcal{N}(\mathbf{T_{i}^{T}Y_{j}}|U_{ij}, \sigma_{i}^{2})$, which is again a standard GP term.
\end{prop}

\begin{rem}\label{rem:induc}
As latent processes have regular single-output GP expressions under the DPN hypothesis, any GP approximation (typically inducing points) can be  "plugged" onto the estimators of proposition \ref{prop:decoupled_post} in order to further reduce computational complexity.
\end{rem}  

\subsection{Interpretation of the projected data and projected noise}
The quantities $\mathbf{\Sigma_{P}}$ and $\mathbf{TY}$ seem to play a pivotal role in the structure of the LMC. The next proposition helps to explains their meaning, mostly restating results from \cite{OILMM}.

\begin{prop}\label{prop:interpretation}
\begin{itemize}
\item $\mathbf{T}$ is a generalized inverse of $\mathbf{H}$: $\mathbf{TH = I_{q}}$. Therefore, $\mathbf{HT}$ is a projection.
\item $\mathbf{TY}$ is a maximum likelihood estimator for $\mathbf{U}$: $\mathbf{TY} = \arg \max_{\mathbf{U}} p(\mathbf{Y|U})$. It is an unbiased estimator: $\mathbb{E}\left[\mathbf{TY|U}\right] = \mathbf{U}$. 
\item $\mathbf{TY}$ is a minimal sufficient statistic of the data $\mathbf{Y}$ for the variable $\mathbf{U}$. In other words, $\mathbf{TY}$ contains all the information required to compute the best possible estimate of $\mathbf{U}$. Consequently, $p(\mathbf{U|Y}) = p(\mathbf{U|TY})$.
\item $\mathbf{T\Sigma T^{T} = \Sigma_{P}}$, so that $\mathbf{TY|U} \sim \mathcal{N}(\mathbf{U_{v}, \, \Sigma_{P} \otimes I_{n}})$
\item The latent processes of \emph{any} LMC model are independent conditionally upon observations if and only if $\mathbf{\Sigma_{P}}$ is diagonal.\\
\end{itemize}
\end{prop}

This can be restated in a more informal way:
\begin{itemize}
    \item $\mathbf{T}$ is a pseudo-inverse of the mixing matrix $\mathbf{H}$, so that $\mathbf{u \simeq Ty}$.
    \item The projected data $\mathbf{TY}$ contains all the information usable by the latent processes.
    \item As $\mathbf{Ty = u + T}\bm{\epsilon}$, $\mathbf{Ty|u} \sim \mathcal{N}(\mathbf{u, \, T\Sigma T^{T} \otimes I_{n}})$ : $\mathbf{\Sigma_{P} = T\Sigma T^{T}}$ is the observational noise projected in the latent space.\\
\end{itemize}

\section{A new model: the Projected Linear Model of Co-regionalization}\label{sec:plmc}

Propositions \ref{prop:decoupled_post} and \ref{prop:decoupled_lik} pave the way for efficient computations with the LMC: if the DPN condition \ref{def:dpn_cond} stands, all GP operations can be performed independently at the level of the little-numerous latent processes. However, enforcing this condition means that the matrices $ \mathbf{H}$ and $\mathbf{\Sigma}$, which are parameters of the model, become coupled and can no longer be optimized independently. In the present section, we derive a parametrization which enforces the DPN condition, and gives rise to a slightly-restricted but computationally efficient implementation of the LMC, which we will name \emph{Projected Linear Model of Co-regionalization}.\\
Enforcing the DPN condition can be done by constraining the mixing matrix $\mathbf{H}$, the noise covariance $\mathbf{\Sigma}$, or a mixture of the two. In this work, we chose to maintain a fully general $\mathbf{H}$, in order to preserve maximal model expressiveness; but other choices could be made to follow different objectives, notably uncertainty propagation in physics simulations.\\

\subsection{DPN-compatible parametrization of the LMC}\label{sec:noise_param}

Reminding that the DPN condition \ref{def:dpn_cond} is written ($\mathbf{H^{T}\Sigma^{-1}H}$ is diagonal), we hint that it can be expressed as some orthogonality condition between rows of $\mathbf{\Sigma^{-1}}$ and columns of $\mathbf{H}$. We therefore introduce the QR decomposition of $\mathbf{H}$: $\mathbf{H = QR}$, where $\mathbf{Q}$ has orthonormal columns and shape $p \times q$, and $\mathbf{R}$ is upper-triangular of size $q \times q $. We also introduce $\mathbf{Q_{\bot}}$, a $p \times  (p - q)$ orthonormal complement of $\mathbf{Q}$ which for now is arbitrary. Using these building blocks, we introduce a decomposition of the precision matrix $\mathbf{\Sigma^{-1}}$ which is the basis for our parametrization:

\begin{prop}\label{prop:facto_Q}
Given matrices $\mathbf{Q}$, $\mathbf{Q_{\bot}}$ and $\mathbf{R}$ defined as above, the following factorization stands for any symmetric matrix $\mathbf{\Sigma^{-1}}$:
\begin{equation}
\mathbf{\Sigma^{-1} = Q_{+}R_{+}^{-T}D_{+}^{-1}R_{+}^{-1}Q_{+}^{T}},
\end{equation}
where: 
$  \mathbf{Q_{+}} =  \left( \begin{array}{c|c} \mathbf{Q} & \mathbf{Q_{\bot}} \end{array} \right) $,
$  \mathbf{R_{+}} =  \left( \begin{array}{c|c} \mathbf{R} & \mathbf{0} \\ \hline \mathbf{0} & \mathbf{I_{p-q}} \end{array} \right) $,
$  \mathbf{D_{+}^{-1}} =  \left( \begin{array}{c|c} \mathbf{\Sigma_{P}^{-1}} & \mathbf{M} \\ \hline \mathbf{M^{T}} & \mathbf{\Sigma_{\bot}^{-1}} \end{array} \right) $, $\mathbf{\Sigma_{P}^{-1} = R^{T}Q^{T} \Sigma^{-1} Q R}$, $\mathbf{\Sigma_{\bot}^{-1}} = \mathbf{Q_{\bot}^{T} \Sigma^{-1} Q_{\bot}}$ and $\mathbf{M = R^{T} Q^{T} \Sigma^{-1} Q_{\bot}}$. All matrices labeled with a '$+$' have size $p\times p$; the notation '$-T$' denotes inverse transpose. Naturally, this definition of $\mathbf{\Sigma_{P}}$ coincides with this of definition \ref{def:T_sigma}, so it is diagonal if and only if the DPN condition stands. 
\end{prop}

\begin{defi}\label{cor:facto_Q}
If a symmetric matrix $\mathbf{\Sigma^{-1}}$ is decomposed as in proposition \ref{prop:facto_Q}, then we can write its inverse as: $\mathbf{\Sigma = Q_{+}R_{+}D_{+}R_{+}^{T}Q_{+}^{T}}$, 
where the blocks of $\mathbf{D_{+}}$ are denoted $\mathbf{D_{+}} = \left( \begin{array}{c|c} \mathbf{\tilde{\Sigma_{P}}} & \mathbf{\tilde{M}} \\ \hline \mathbf{\tilde{M}^{T}} & \mathbf{\tilde{\Sigma}_{\bot}} \end{array} \right) $ and can be computed from blockwise inversion formulas.\\
\end{defi}

This decomposition and definition are sufficient to introduce a DPN-compatible parametrization of the LMC:

\begin{prop}\label{prop:noise_param}
$\mathbf{\tilde{\Sigma}_{\bot}}, \mathbf{\Sigma_{P}}$ and $\mathbf{M}$ form a full parametrization of $\mathbf{\Sigma}$ and can be optimized independently. $\mathbf{\Sigma}$ is enforced to be p.s.d (positive semi-definite) if and only if $\mathbf{\Sigma_{P}}$ and $\mathbf{\tilde{\Sigma}_{\bot}}$ are respectively positive and p.s.d.\\
\end{prop}

This parametrization also yields valuable insight about the model. For instance, we have the following expression for $\mathbf{T}$, which connects it to the pseudoinverse of $\mathbf{H}$: 

\begin{prop}\label{prop:express_T}
\begin{equation}\label{eq:express_T}
\mathbf{T} = \mathbf{H^{+} + \Sigma_{P}MQ_{\bot}^{T}}
\end{equation}
where $\mathbf{H^{+} = R^{-1}Q^{T}}$ is the Moore-Penrose pseudoinverse of $\mathbf{H}$. This expression doesn't depend on a specific choice of the orthogonal complement $\mathbf{Q_{\bot}}$.\\
\end{prop} 

\subsection{The Projected Linear Model of Co-regionalization}\label{sec:mll}

\begin{prop}\label{prop:final_mll}
Under the DPN condition, we have the following decoupled expression for the marginal log-likelihood (which doesn't depend on a particular choice of the orthogonal complement $\mathbf{Q_{\bot}}$): 
\begin{multline}\label{eq:final_mll}
- 2 \log p(\mathbf{Y}) = (p-q)n\log 2\pi \, + \, 2n\log |\mathbf{R}| \, + \, n\log |\mathbf{\tilde{\Sigma}_{\bot}}| \\
 + \text{Tr}(\mathbf{Y^{T} Q_{\bot}\tilde{\Sigma}_{\bot}^{-1} Q_{\bot}^{T} Y}) + \sum_{i=1}^{q} \log \mathcal{N}(\mathbf{T_{i} Y| 0, \, K_{i} \, + \, } \sigma_{i} \mathbf{I_{n}}) 
\end{multline}
The term $\mathbf{\tilde{\Sigma}_{\bot}}$ can thus be interpreted as a \emph{discarded noise}: it is absent from all latent-related quantities, but required in the task-level likelihood to account for the variability of observations.
\end{prop}

We can finally give a full specification of our proposed model, after noticing a redundancy in our parametrization:
\begin{prop}\label{prop:diagonal_B}
The projection matrix $\mathbf{T}$ and posteriors of the model do not depend on $\mathbf{\tilde{\Sigma}_{\bot}}$; moreover, $\mathbf{\tilde{\Sigma}_{\bot}}$ can be restricted to be diagonal without incidence on the likelihood.
\end{prop}

\begin{defi}\label{def:PLMC}
	The \emph{Projected Linear Model of Co-regionalization} -- abbreviated as PLMC -- is a LMC model enforcing the DPN condition (definition \ref{def:dpn_cond}).  Its parameters are:

\begin{itemize}
\item An orthonormal $p \times p$ matrix $\mathbf{Q_{+}}$ (joint optimization of $\mathbf{Q}$ and $\mathbf{Q_{\bot}}$);
\item An upper-triangular $q \times q$ matrix $\mathbf{R}$;
\item A $(p-q) \times (p-q)$ lower-triangular, positive-diagonal matrix $\mathbf{L}$, Cholesky factor of $\mathbf{\tilde{\Sigma}_{\bot}}^{-1}$, which by proposition \ref{prop:diagonal_B} can be assumed diagonal;
\item A $q \times (p-q)$ matrix $\mathbf{M}$ with no specific property; 
\item A $q \times q$ positive diagonal matrix $\mathbf{\Sigma_{P}} = Diag(\sigma_{i})_{1\leq i \leq q}$; 
\item All parameters of the kernels $k_{i}$.
\end{itemize}

It can be computed in the following way:
	\begin{align*}\label{eq:PLMC}
		\mathbf{H} &= \mathbf{QR} \\
		\mathbf{T} &= \mathbf{R^{-1}Q^{T} + \Sigma_{P}MQ_{\bot}^{T}}\\
		\mathbf{\tilde{\Sigma}_{\bot}}^{-1} &= \mathbf{LL^{T}} \\
		\mathbf{\Sigma} &= \mathbf{H(\Sigma_{P} + \Sigma_{P} M \tilde{\Sigma}_{\bot} M^{T} \Sigma_{P})H^{T}} - 2 \mathbf{Sym(H \Sigma_{P} M \tilde{\Sigma}_{\bot} Q_{\bot}^{T})} + \mathbf{Q_{\bot} \tilde{\Sigma}_{\bot}^{-1} Q_{\bot}^{T}}\\
		\hat{\mathbf{y}}_{*} &= \mathbf{H} \hat{\mathbf{u}}_{*}\\
		p(\hat{u}_{i*}|\mathbf{Y}) &= \mathcal{N} ( \mathbf{k_{i*}^{T} (K_{i}} + \sigma_{i}^{2} \mathbf{I_{n})^{-1} T_{i} Y}, \ k_{i**} - \mathbf{k_{i*}^{T} (K_{i}} + \sigma_{i}^{2} \mathbf{I_{n})^{-1}k_{i*}} ) \quad \forall i \in [1, q] \\
		- 2 \log p(\mathbf{Y}) &= (p-q)n\log 2\pi \, + \, 2n\log |\mathbf{R}| \, + \, n\log |\mathbf{\tilde{\Sigma}_{\bot}}| \\
        & \quad + \text{Tr}(\mathbf{Y^{T} Q_{\bot}\tilde{\Sigma}_{\bot}^{-1} Q_{\bot}^{T} Y}) + \sum_{i=1}^{q} \log \mathcal{N}(\mathbf{T_{i} Y| 0, \, K_{i} \, + \, } \sigma_{i} \mathbf{I_{n}})
	\end{align*}

with $\mathbf{Sym}(\cdot)$ the symmetric part of a matrix: $\mathbf{Sym}(\mathcal{M}) = \frac{\mathcal{M} + \mathcal{M}^{T}}{2}$.\\
\end{defi}

\begin{rem}
  Note that enforcing the DPN condition while maintaining a fully general $\mathbf{H}$ resulted in a complex noise model $\mathbf{\Sigma}$; in particular, $\mathbf{\Sigma}$ can no longer be chosen diagonal(independent per-task noises).
\end{rem}

\begin{rem}\label{rem:oilmm}
In \cite{OILMM}, the mixing matrix of the OILMM is written $\mathbf{H = QS^{1/2}}$ with $\mathbf{Q}$ orthonormal and $\mathbf{S}$ positive diagonal, and its noise model is $\mathbf{\Sigma} = \sigma \mathbf{I + HDH^{T}} $ with $\mathbf{D}$ diagonal. If we rewrite the latter $\mathbf{\Sigma = Q_{\bot}} (\sigma \mathbf{I) Q_{\bot}^{T}} + \mathbf{Q S^{1/2} (D + } \sigma \mathbf{S^{-1}) S^{1/2} Q^{T} }$, and make the following identifications: 

\begin{align}
\mathbf{S^{1/2}} & \mathbf{\equiv R}\\
\mathbf{D} + \sigma \mathbf{S^{-1}} &\equiv \mathbf{\tilde{\Sigma}_{P}}\\ \sigma \mathbf{I} &\equiv \mathbf{\tilde{\Sigma}_{\bot}} \quad \text{,}
\end{align}

we see that the OILMM is a special case of our model, where $\mathbf{R}$ is constrained to be positive diagonal instead of triangular, $\mathbf{M}$ is zero and $\mathbf{\tilde{\Sigma}_{\bot}}$ is scalar instead of symmetric (or diagonal, see proposition \ref{prop:diagonal_B}).\\
Looking in details into the GPPCA (\cite{GPPCA}), it can also be seen that it is an OILMM model further restricted to have a scalar $\mathbf{\Sigma}$, which implies $\mathbf{S = I_{q}}$ and $\mathbf{D=0}$.\\
\end{rem}

\subsection{A faster and simpler variant: the PLMC-fast}\label{sub:PLMC_fast}

The above remark frames OILMM as a special case of the PLMC. In practice, its implementation in the original article \cite{OILMM} has a major benefit over the PLMC definition of definition \ref{def:PLMC}: the latter features a full $p \times p$ orthonormal matrix $\mathbf{Q_{+}}$, whereas the former only involves a $p \times q$ matrix $\mathbf{Q}$. This is likely to have a large impact on computation costs: the cost of updating a parametrized $p \times p$ orthogonal matrix is $\mathcal{O}(p^{3})$, while this for a $p\times q$ one is $\mathcal{O}(pq^{2})$ \cite{OILMM}. One should therefore wonder whether the additional parameter $\mathbf{Q_{\bot}}$ of the PLMC could be spared.\\

Looking at the expressions of the PLMC, we see that $\mathbf{Q_{\bot}}$ appears in two places only: the expression of $\mathbf{T}$ in eq.\eqref{eq:express_T}, and a regularizing term in eq.\eqref{eq:final_mll}. The first occurrence can be suppressed by making the simplification $\mathbf{M=0}$ -- no correlation between the projected noise $\mathbf{\Sigma_{P}}$ and discarded noise $\mathbf{\Sigma_{\bot}}$. The second occurrence is more subtle: looking closely at the MLL term $\text{Tr}(\mathbf{Y^{T} Q_{\bot}\tilde{\Sigma}_{\bot}^{-1} Q_{\bot}^{T} Y})$, it appears that the only way to make this term independent on $\mathbf{Q_{\bot}}$ is to enforce $\mathbf{\tilde{\Sigma}_{\bot}}$ to be scalar. Indeed, if $\mathbf{\tilde{\Sigma}_{\bot}} = \sigma_{\bot}\mathbf{I}_{p-q}$ with $\sigma_{\bot}$ a positive scalar, then:

\begin{equation}
\text{Tr}(\mathbf{Y^{T} Q_{\bot}\tilde{\Sigma}_{\bot}^{-1} Q_{\bot}^{T} Y}) = \sigma_{\bot}^{-1} \text{Tr}(\mathbf{Y^{T} Q_{\bot}Q_{\bot}^{T} Y}) = \sigma_{\bot}^{-1} ||\mathbf{Y}||_{F}^{2} \quad \text{,}
\end{equation}

Conversely, it is clear that only a scalar matrix $\mathbf{\tilde{\Sigma}_{\bot}}$ can make the quantity $\text{Tr}(\mathbf{Y^{T} Q_{\bot}\tilde{\Sigma}_{\bot}^{-1} Q_{\bot}^{T} Y})$ independent on $\mathbf{Q_{\bot}}$. We can therefore define the simplified model PLMC-fast in the following manner:

\begin{defi}\label{def:PLMC_fast}
	The PLMC-fast is a variant of the PLMC where the following assumptions are made: $\mathbf{M=0}$ and $\mathbf{\tilde{\Sigma}_{\bot}} = \sigma_{\bot}\mathbf{I}_{p-q}$ with $\sigma_{\bot}$ a positive scalar. Its parameters are: a $p \times q $ orthonormal matrix $\mathbf{Q}$, a $q \times q$ upper-triangular matrix $\mathbf{R}$ with positive diagonal, a $q \times q$ positive diagonal matrix $\mathbf{\Sigma_{P}} = Diag(\sigma_{i})_{1\leq i \leq q}$, a positive scalar $\sigma_{\bot}$, and the parameters of the latent kernels $k_{i}$. It obeys the following equations:

	\begin{align*}
		\mathbf{H} &= \mathbf{QR} \\
		\mathbf{T} &= \mathbf{R^{-1}Q^{T}} = \mathbf{H^{+}}\\
		\mathbf{\Sigma} &= \mathbf{H\Sigma_{P}H^{T}} + \sigma_{\bot}(\mathbf{I - QQ^{T}})\\
		\hat{\mathbf{y}}_{*} &= \mathbf{H} \hat{\mathbf{u}}_{*}\\
		p(\hat{u}_{i*}|\mathbf{Y}) &= \mathcal{N} ( \mathbf{k_{i*}^{T} (K_{i}} + \sigma_{i}^{2} \mathbf{I_{n})^{-1} T_{i} Y}, \ k_{i**} - \mathbf{k_{i*}^{T} (K_{i}} + \sigma_{i}^{2} \mathbf{I_{n})^{-1}k_{i*}} ) \quad \forall i \in [1, q] \\
		- 2 \log p(\mathbf{Y}) &= (p-q)n\log (2\pi + \sigma_{\bot}) \, + \, 2n\log |\mathbf{R}| \, + \, \sigma_{\bot}^{-1} ||\mathbf{Y}||_{F}^{2} \\
        & \quad + \, \sum_{i=1}^{q} \log \mathcal{N}(\mathbf{T_{i} Y| 0, \, K_{i} \, + \, } \sigma_{i} \mathbf{I_{n}}) 
	\end{align*}

\end{defi}

\begin{rem}\label{rem:simplif}
Notice that in the full PLMC, $\mathbf{M}$ and $\mathbf{\tilde{\Sigma}_{\bot}}$ are noise parameters, which are likely to be of low magnitude and low relevance to the model: $\mathbf{\tilde{\Sigma}_{\bot}}$ is a "discarded noise" which is absent from the posteriors, while $\mathbf{M}$ couples the projected noise with the discarded noise. The simplifications made in the PLMC-fast are thus likely to have little effect on the performance of the model, which will be verified in the experimental section.
\end{rem}

At this point, one can wonder in which manner the PLMC-fast differs from the OILMM. The answer is simple: its matrix $\mathbf{R}$ is triangular rather than diagonal, meaning that its mixing matrix $\mathbf{H}$ is fully general instead of orthogonal. It turns out that the restriction on $\mathbf{H}$ imposed by the OILMM was not necessary to the coherence of the model: a fully general mixing matrix was available for free by going through a different derivation. This brings about an important remark about PLMC implementations:

\begin{rem}
In the PLMC, the parameters $\mathbf{Q_{+}}$ and $\mathbf{R_{+}}$ (see proposition \ref{prop:facto_Q}) make for the QR decomposition of a $p \times p$ matrix -- let us call it $\mathbf{H_{+}}$ -- which first $q$ columns are the mixing matrix $\mathbf{H}$. This matrix $\mathbf{H_{+}}$ is fully generic; therefore, rather than parametrizing an orthonormal matrix $\mathbf{Q_{+}}$, one can instead parametrize a standard matrix $\mathbf{H_{+}}$, and recover $\mathbf{Q_{+}}$ and $\mathbf{R}$ from it\footnote{$\mathbf{R}$ is then simply the upper-left $q \times q$ block of $\mathbf{R_{+}}$.} using an auto-differentiable QR decomposition algorithm such as this of \cite{autodiff_decomp}. This has significant applicative interest, as all existing parameterizations of orthonormal matrices exhibit shortcomings in some settings\footnote{The orthonormal matrix parameterizations known to us are the \emph{matrix exponential transform}\cite{matexp_transfo}, the \emph{Cayley transform} \cite{Cayley_transfo}, and the product of Householder reflections \cite{Householder_param}; the latter two are known to be imperfectly stable \cite{matexp_transfo}, and the first also was in our experiments. One more advanced approach is this of \cite{ortho_param}, implemented in \texttt{torch}; it is very stable, but requires the storage of a $p\times p$ matrix buffer, making PLMC training infeasible for large values of $p$.}. By contrast, doing the same with the OILMM would require mapping the triangular matrix $\mathbf{R}$ to a diagonal one, resulting in severe over-parametrization and probable instability. With the PLMC-fast, the situation is even simpler: $\mathbf{Q}$ and $\mathbf{R}$ can directly be computed as the QR decomposition of  $\mathbf{H}$.
\end{rem}

\section{Experimental results}\label{experiments}

\subsection{Tested models}\label{sec:test_models}

We compare our PLMC and PLMC-fast to several baselines: a state-of-the art ICM model, a state-of-the art variational LMC, and the OILMM - which in fact is only a subcase of PLMC as explained in remark \ref{rem:oilmm}. The five implemented models are thus denoted \texttt{ICM} (LMC with a kernel shared by all latent processes and algebra tricks -- see \cite{all_in_the_noise} for the computation implemented in \texttt{gpytorch}), \texttt{var} (variational LMC as described in \cite{hensman_class} and implemented in \texttt{gpytorch} (\cite{gpytorch}), with a Cholesky variational distribution and a number $\lfloor n/1.5\rfloor $ of learned inducing points\footnote{Picking a larger number of learned inducing points or even putting fixed inducing points at all observed locations, although theoretically more accurate, proved less stable in practice.}), \texttt{PLMC} (fully parametrized PLMC as in definition \ref{def:PLMC}, with a triangular matrix $\mathbf{L}$), \texttt{PLMC_fast} (as described in definition \ref{def:PLMC_fast}) and \texttt{oilmm} (implemented as a PLMC-fast with a diagonal matrix $\mathbf{R}$, which is equivalent to the implementation of the original article \cite{OILMM}). Unless otherwise specified, the \texttt{ICM} and \texttt{var} models are implemented with a fully general gaussian likelihood, putting no constraint on their $\mathbf{\Sigma}$ matrix.\\
The training protocol for each model (optimizer, learning rate schedule, number of iterations, initialization, stopping criterion...) was hand-refined and finally taken identical for all of them, as the chosen common scheme yielded near-optimal convergence for all; it is specified in \ref{an:exp_spec}.

\subsection{Synthetic data}\label{sec:synth}

\paragraph{Data description}
We start with a parametric study on synthetic data, in order to explore the effect of some data properties on our models. It is generated according to a univariate LMC model: its latent processes are defined by $q$ 1D Matern kernels ($\nu = 2.5$) with lengthscales equidistant in the interval $[l_{min}, \, l_{max}]$, and sampled at the locations $\mathbf{X_{train}}$ ($n$ points equidistant in $[-1, \, 1]$) and $\mathbf{X_{test}}$ (2500 points sampled uniformly in the same interval). The coefficients of the mixing matrix $\mathbf{H}$ are sampled from independent normal distribution $\mathcal{N}(0,1)$. After latent signals are mixed, structured noise is added in a manner inspired from \cite{all_in_the_noise}: a parameter $\mu_{noise}$ controls the signal-to-noise proportion ($\mathbf{Y}_{full} = \mu_{noise} \, \mathbf{Y}_{noise} + (1 - \mu_{noise}) \, \mathbf{Y}_{signal} $), and another $\mu_{str}$ rules the proportion of structured-to-unstructured noise ($\mathbf{Y}_{noise} = \mu_{str} \, \mathbf{Y}_{noise}^{str} \, + \, (1 - \mu_{str}) \, \mathbf{Y}_{noise}^{ind} $). $\mathbf{Y}_{noise}^{str}$ is obtained by the mixing of $q_{noise}$ white noise processes by a matrix $\mathbf{H_{noise}}$, generated in the same way as $\mathbf{H}$, while $\mathbf{Y}_{noise}^{ind}$ is simply the concatenation of $p$ independent white noises. The noise thus contains a part that is correlated over all tasks, and one that is specific to each task. Overall the synthetic data is described by 8 parameters; when they are not varied for an experiment or explicitly specified, their fixed values are the following: $p=100$, $q=25$, $q_{noise}=25$, $n=500$, $\mu_{noise}=0.1$, $\mu_{str}=0.9$, $l_{min}=0.01$ and $l_{max}=0.5$.\\
All tests are run $N_{rep}$ times with data generated from as many random seeds, and resulting metrics are averaged. For all models, the chosen kernel type and number of latent processes are always this of the data (no model mis-specification).\\

\paragraph{DPN condition severity} 
The first and main observation is that \textbf{projected models are overall as much or more accurate than the baselines}, even in setups where high-magnitude structured noise should handicap them. This way, figure \ref{fig:err_mu_noise} shows that for the highest noise magnitudes ($\mu_{noise} \simeq 0.5$, corresponding to a signal-to-noise ratio of 1) and a highly-structured noise ($\mu_{str}=0.99$), the RMSEs of the PLMC and PLMC-fast are as low as these of their ICM and variational counterparts: in fact, all of them follow the expected asymptote $RMSE \simeq \mu_{noise}$. Before this asymptote, the \texttt{ICM} is struggling; its kernel structure may not be flexible enough to model this synthetic data, which features $q$ different lengthscales (see the above paragraph). For low noise values (less than $1\%$), the variational model is less accurate than the PLMC variants; also notice one catastrophic non-convergence of the OILMM, likely due to the stability issues of orthogonal matrix parameterizations mentioned in section \ref{sub:PLMC_fast}, which polluted one datapoint even when averaged over $50$ random experiments. The same conclusion stands when varying $\mu_{str}$ (figure \ref{fig:err_mu_str}) or $q_{noise}$ (not displayed): projected models don't suffer disproportionately from complex noise structure in the data.\\
Interestingly, the simplified PLMC models (PLMC-fast and OILMM) seem to be more robust than the full one, yielding consistently higher predictive accuracies. There are two possible explanations for this: either they benefit from a more favorable optimization landscape (fewer parameters, optimization of a $p\times q$ rather than $p\times p$ orthonormal matrix), or the full PLMC mistakes signal for noise because of its more expressive noise model. Figure \ref{fig:pva} seems to refute this last hypothesis, as the PLMC yields better-calibrated predictive variance than the PLMC-fast and OILMM accross the whole $\mu_{str}$ range.

\begin{figure}[htbp]
    \centering
    \begin{subfigure}[htbp]{0.55\paperwidth}
      \centering
\includegraphics[width=1.3\textwidth]{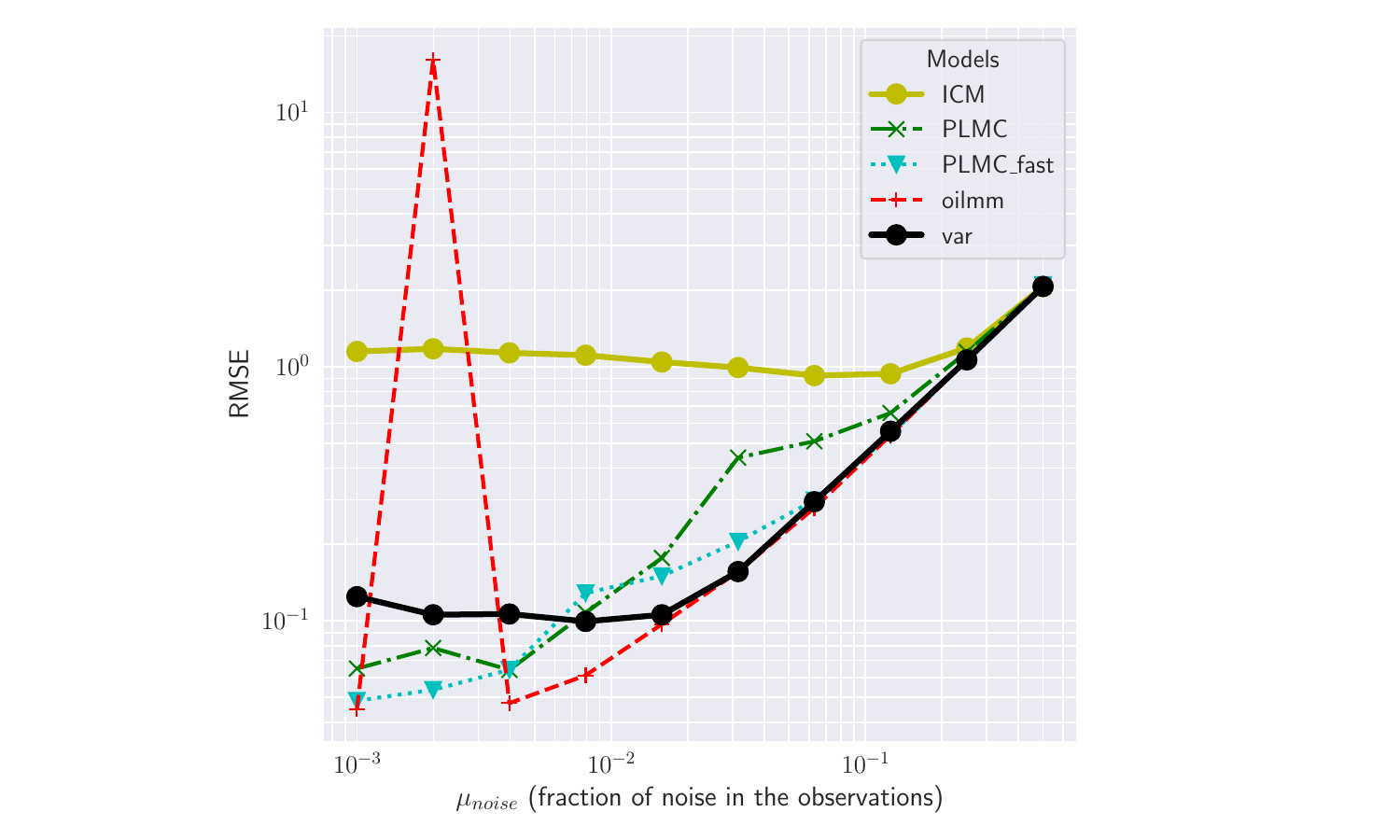}
\caption{Varying $\mu_{noise}$. With $\mu_{str}=0.99$, $N_{rep}=50$}
\label{fig:err_mu_noise}
    \end{subfigure}
    \quad
    \begin{subfigure}[htbp]{0.56\paperwidth}
      \centering
\includegraphics[width=\textwidth]{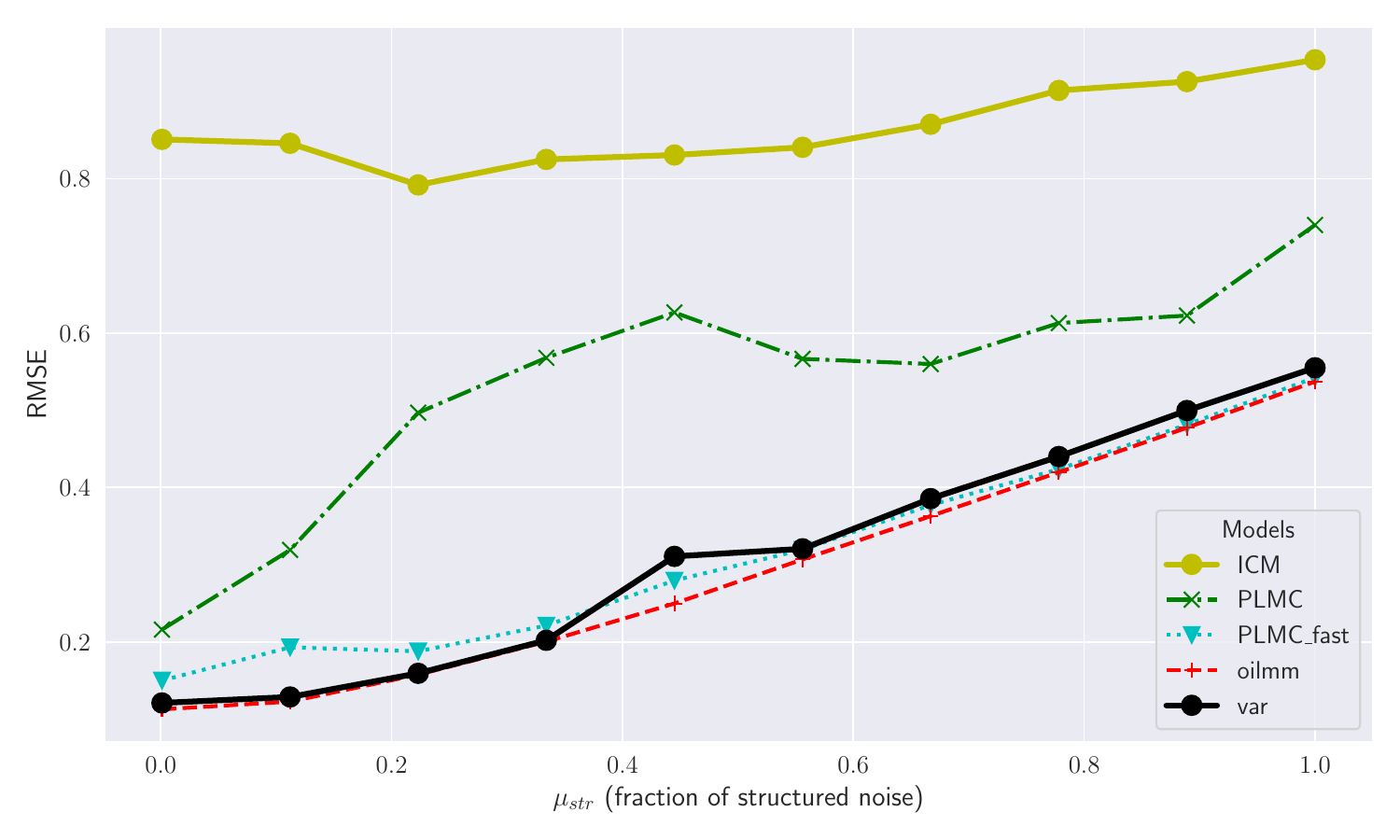}
\caption{Varying $\mu_{str}$. With $\mu_{noise}=0.1$, $N_{rep}=40$}
\label{fig:err_mu_str}
\end{subfigure}
\caption{RMSE of several models for increasing data noise magnitude (a) and proportion of structured noise (b).}
\label{fig:errs}
\end{figure}

\textbf{The impact of noise model restrictions on variance estimation is also surprisingly weak}. We measure the quality of variance estimation through Predictive Variance Adequacy, which assesses whether the predictive distribution has the right width: $PVA = \frac{1}{p} \sum_{j=1}^{p} \log \left( \frac{1}{n} \sum_{i=1}^{n} \frac{(y_{ij} - \hat{y}_{ij})^{2}}{\hat{v}_{ij}}\right)$, with $y_{ij}$ the true value of output $j$ at point $i$, $\hat{y}_{ij}$ the corresponding estimated value and $\hat{v}_{ij}$ the corresponding predicted variance. Its optimal value is $0$, in which case the predictive variance is well-calibrated (equal on average to the actual magnitude of errors). Figures \ref{fig:pva_nlatnoise} and \ref{fig:pva_mustr} -- which do not feature the ICM, as its predictive variance has a memory footprint in $O(n^{2}p^{2})$ in the \texttt{gpytorch} implementation, infeasible for data of this size -- show no problematic behavior for the variance estimates of the PLMC models, even for complex, highly-correlated noises (large values of $q_{noise}$ or $\mu_{str}$). This is unexpected, as the PLMC puts constraints on the inter-task covariance of noises (DPN condition, definition \ref{def:dpn_cond}) and the PLMC-fast and OILMM even more (definition \ref{def:PLMC_fast} and remark \ref{rem:oilmm}). The only observations matching this intuition are the relative underperformance of the OILMM at low values of $q$ (fig.\ref{fig:pva_nlatnoise}) and the overperformance of the full PLMC in fig.\ref{fig:pva_mustr} relative to its simplified versions. Apart from this, it is encouraging that the PLMC models have PVA values as good or better as the \texttt{var} baseline on most of the parameters space.

\begin{figure}[hbtp]
    \centering
    \begin{subfigure}[hbtp]{0.65\paperwidth}
      \centering
      \includegraphics[width=\linewidth]{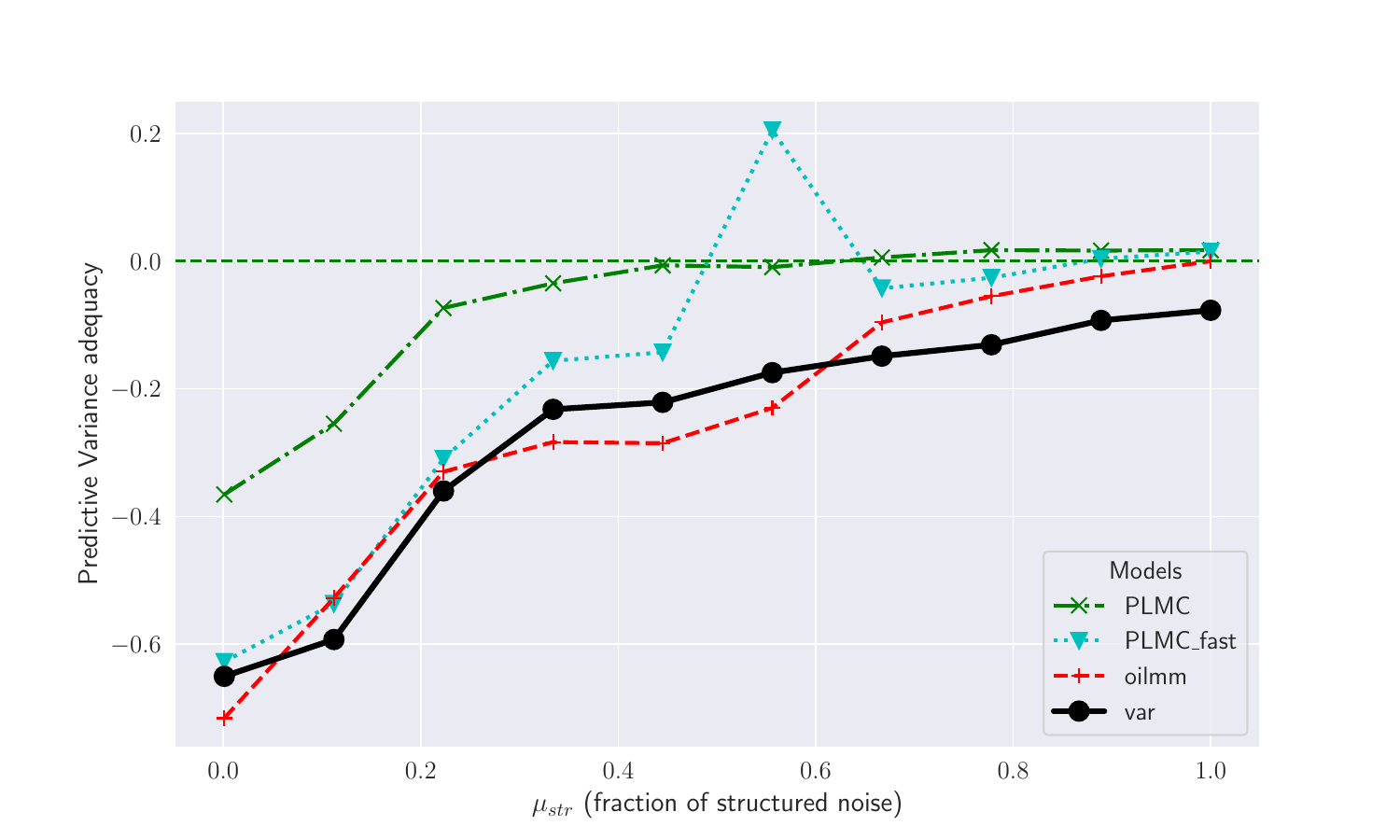}
      \caption{Varying $\mu_{str}$, $q_{noise}=25$}
      \label{fig:pva_mustr}
    \end{subfigure}
    \quad
    \begin{subfigure}[hbtp]{0.6\paperwidth}
      \centering
      \includegraphics[width=\linewidth]{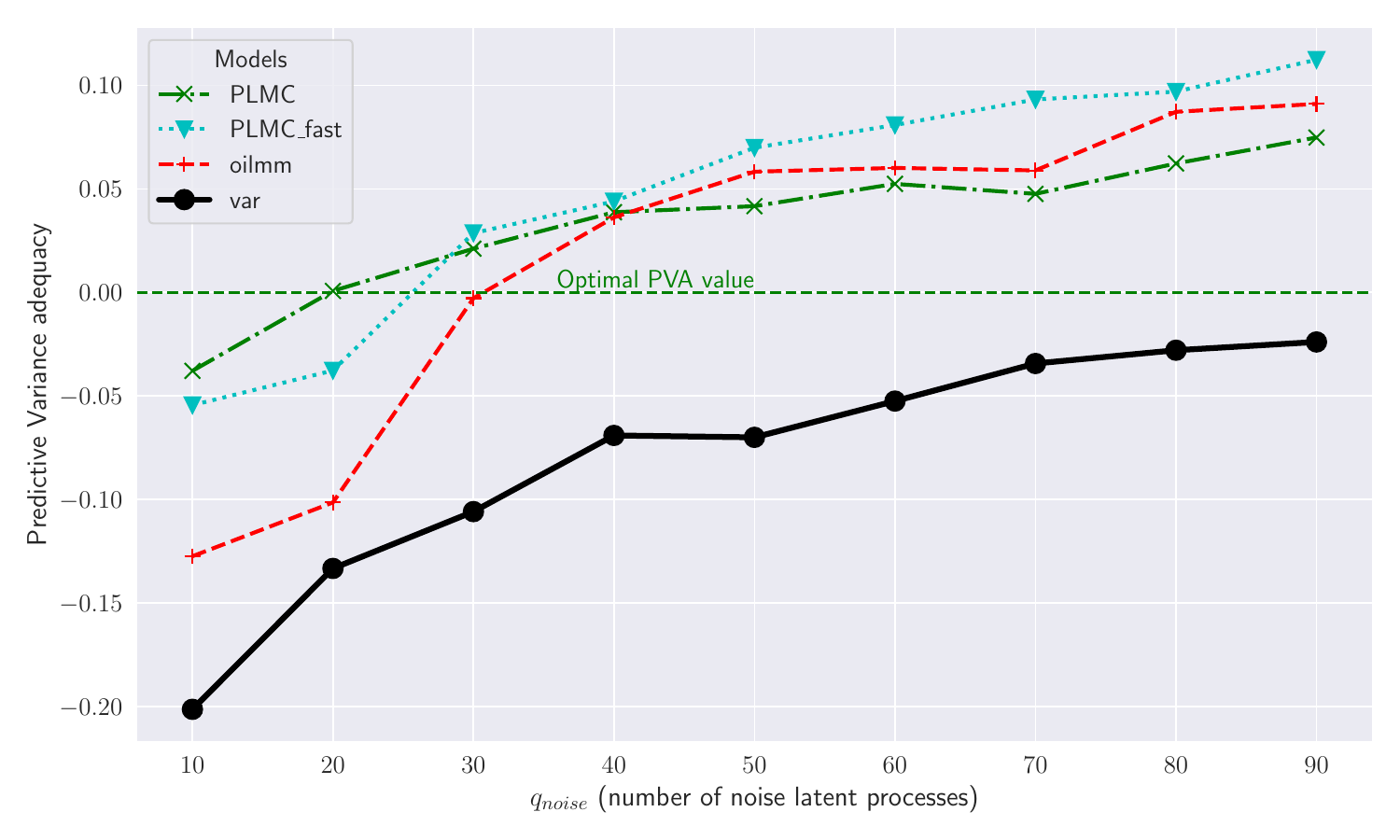}
      \caption{Varying $q_{noise}$, $\mu_{str}=0.9$}
      \label{fig:pva_nlatnoise}
    \end{subfigure}
{\centering
\caption{Average Predictive Variance Adequacy of several models for increasingly structured (a) and complex (b) noises, with $\mu_{noise}=0.1$, and $N_{rep}=10$.}
\label{fig:pva}}
\end{figure}

\paragraph{Training duration} In figure \ref{fig:train_time}, the number of regression tasks $p$ was varied from $50$ to $500$ ; the training durations of models fitted to this data are plotted against $p$. Results yield several valuable observations. First, \textbf{PLMCs (including the OILMM) are the fastest model of all tested}: they best the variational one over the full range of the experiment\footnote{Excepting for the first point of the \texttt{PLMC-fast}, which seems to be associated with a convergence accident.}. The ICM is not even displayed on this graph for being too far above the others, its training time ranging from $200\,$s to $1000\,$s.\\

\begin{figure}[hbt]
\centerline{\includegraphics[width=\textwidth]{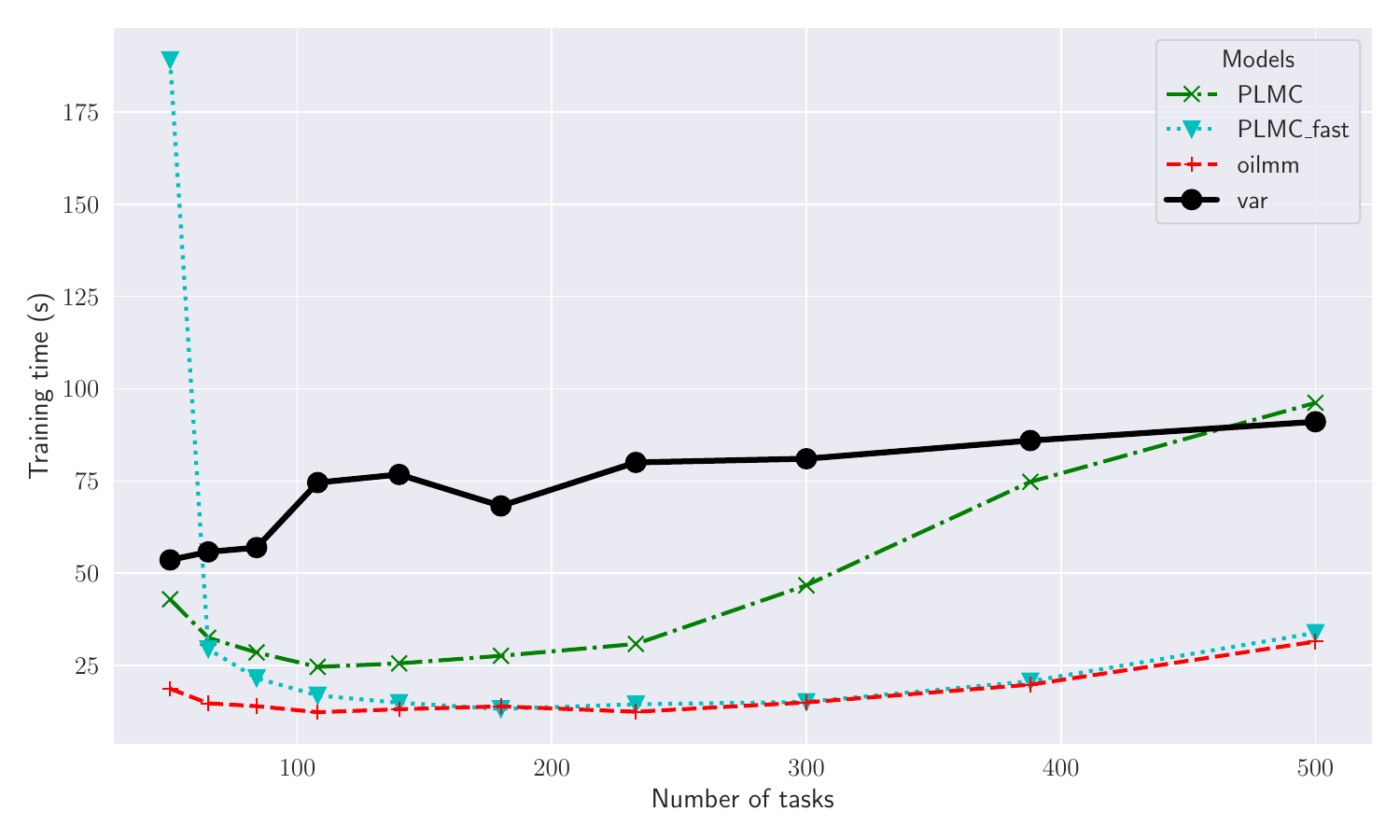}}
\caption{Training duration of several models for increasing number of tasks and a fixed number $q=25$ of latent processes. Averaged over $N_{rep}=50$ random datasets}
\label{fig:train_time}
\end{figure}

Secondly, \textbf{the training durations of the simplified PLMC models -- PLMC-fast and OILMM -- scale slower with $\mathbf{p}$ than this of the full model}. This behavior was expected, and is likely due to having only a $p \times q$ orthonormal matrix to optimize instead of a $p \times p$ one: it can be seen in table \ref{tab:comp_comp} of \ref{an:complexities} that the cost of one evaluation of the MLL is $\mathcal{O}(qn^{3} + p^{3})$ for the full PLMC, but $\mathcal{O}(qn^{3} + pq^{2})$ for PLMC-fast and OILMM (as well as $\mathcal{O}(qm^{3} + qnm^{2} + npq)$ for the variational LMC, with $m$ the number of inducing points). We can take this analysis one step further by plotting the duration of one training iteration $T_{train}/N_{iter}$ rather than the total duration $T_{train}$, minding that $N_{iter}$ is variable in our case, yielded by the stopping criterion described in \ref{an:exp_spec}. This is done in figure \ref{fig:Titers}: the main observation is that for all PLMC models and the variational LMC, the duration of a training iteration scales much more sharply (but still linearly) with $q$ than with $p$. This matches the observation made about the OILMM in \cite{OILMM}: the cost $\mathcal{O}(pq^{2})$ associated with the update of $\mathbf{H}$ is often negligible compared to this of operations with latent processes ($\mathcal{O}(qn^{3})$). The ICM, on the other hand, scales sharply with $p$ and not at all with $q$; this is coherent with the theoretical cost of MLL computations for this implementation, which is $\mathcal{O}(p^{3} + n^{3})$ \cite{all_in_the_noise}.

\begin{figure}[p]
    \centering
    \begin{subfigure}[b]{0.6\paperwidth}
        \centering
        \includegraphics[width=\textwidth]{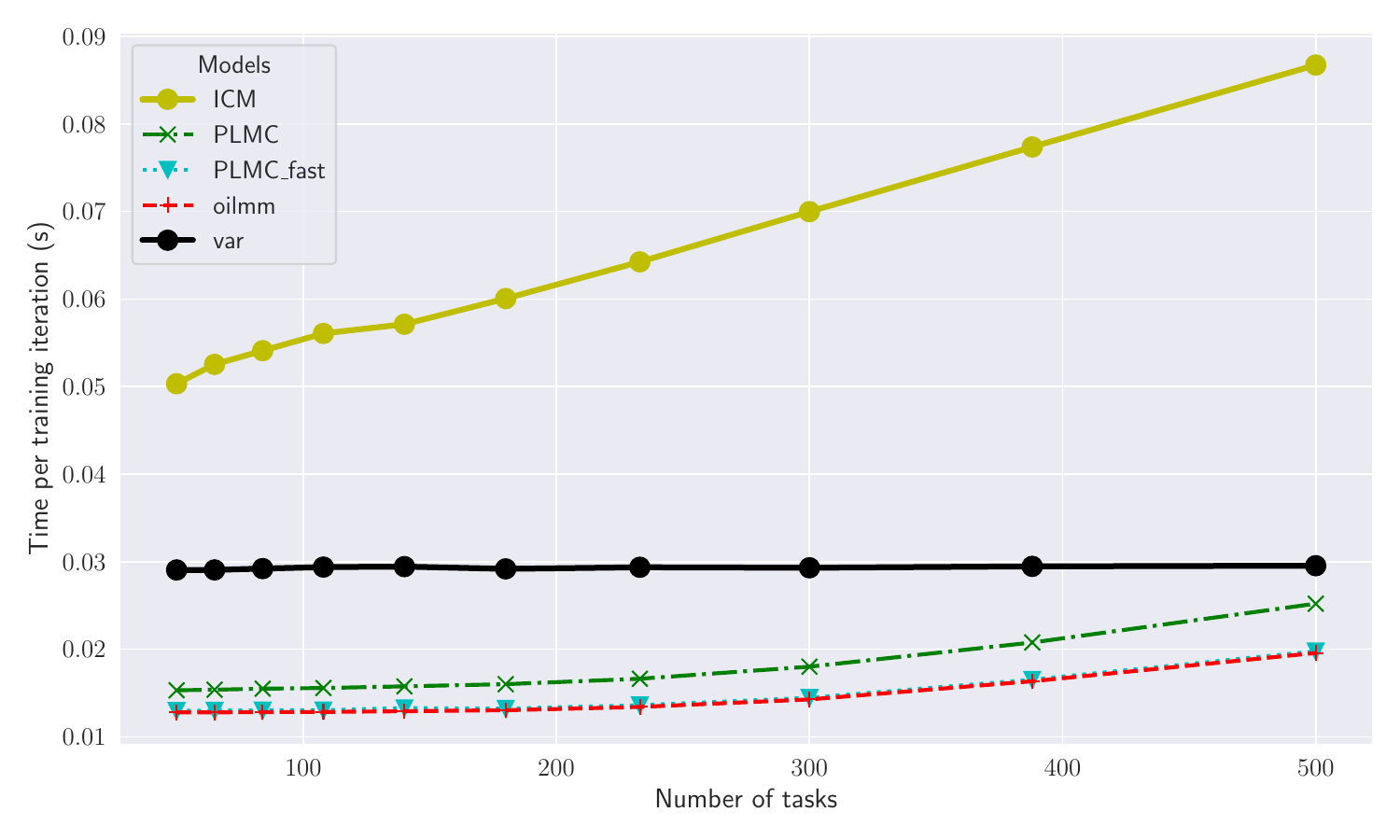}
        \caption{Variable $p$, fixed $q=25$}
        \label{fig:Titer_ntasks}
    \end{subfigure}
    \quad
    \begin{subfigure}[b]{0.6\paperwidth}
        \centering
        \includegraphics[width=\textwidth]{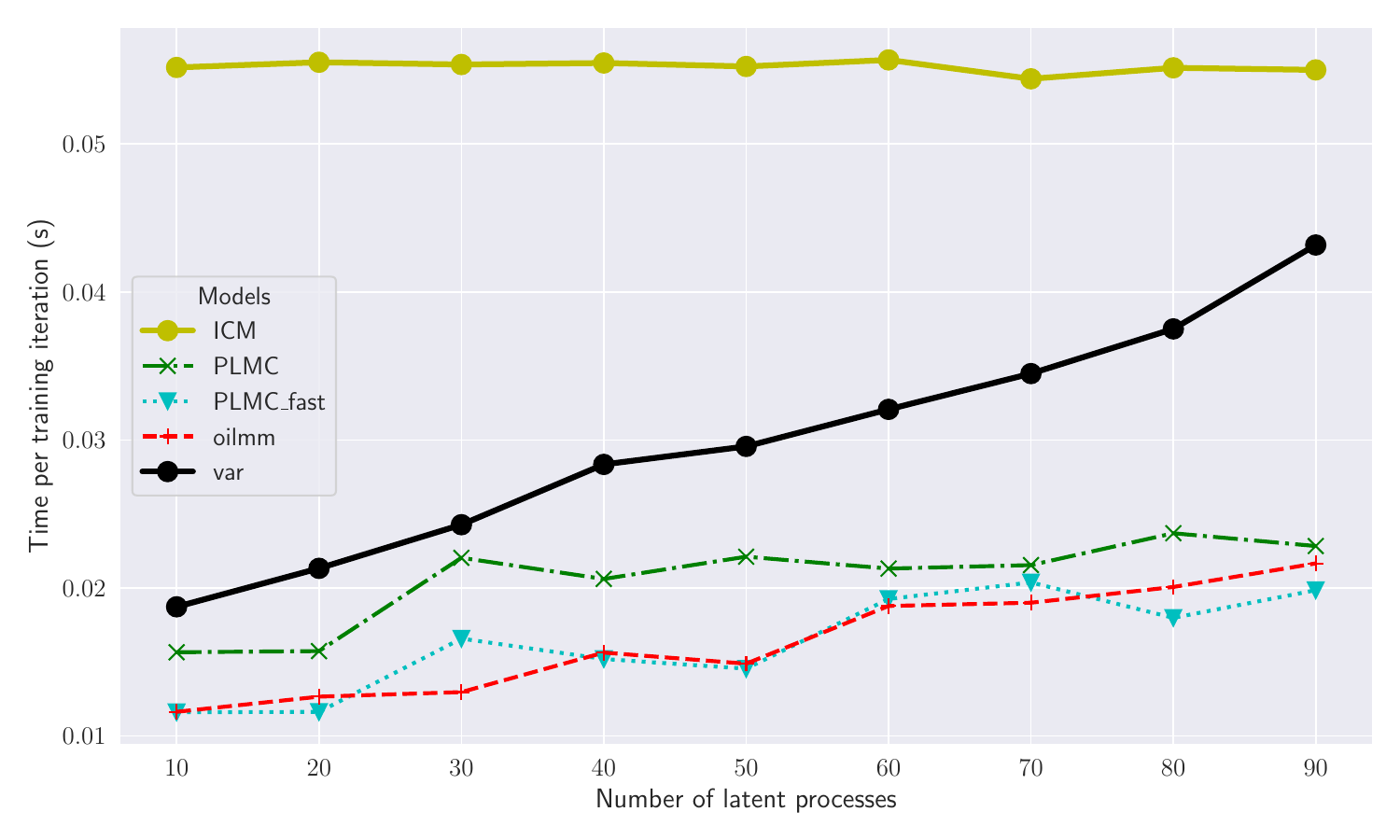}
        \caption{Variable $q$, fixed $p=100$}
        \label{fig:Titer_nlat}
    \end{subfigure}
    \caption{Duration of a training iteration for several models, with a) increasing number of tasks and fixed number of latent processes , and b) the opposite. Averaged over $N_{rep}=20$ random datasets}
    \label{fig:Titers}
\end{figure}

\paragraph{Other observations}
A few other observations on the parametric studies are given in \ref{an:extra_exp}.

\subsection{Real data}\label{sec:real_data}

In the two first experiments, we used a Matérn-5/2 kernel with parametrized lengthscales for all models. The likelihoods of the \texttt{var} and \texttt{ICM} models were taken to be task-independent instead of generic gaussian (diagonal rather than generic $\mathbf{\Sigma}$ matrix) as this choice yielded better predictive performance. Results display the metric $Q^{95}_{L1}$, quantile of L1 errors at level 95\%, in the aim of assessing the occurrence of large errors. An additional experiment is presented in \ref{an:extra_exp}, with coherent results.

\paragraph{SARCOS dataset}
This well-known dataset described in \cite{SARCOS} displays the dynamics of a robot arm (position, speed and acceleration of its 7 articulation points) and the corresponding torque commands for these articulations. The goal is to model its inverse dynamics, i.e to infer the torque commands (7 outputs) from the dynamic measurements (21 features). The dataset contains about $50,000$ points corresponding to various movements sampled at 100Hz; for ease of computation, we subsampled it at 10Hz. We respected the train/test split of \cite{Rasmussen} (yielding about $45,000$ training points and $4,500$ for test), and classically standardized the data. Even after subsampling, the problem remained too large to reach a reasonable computation time; we therefore wrapped latent processes of PLMCs with an inducing points approximation, as suggested in remark \ref{rem:induc}. We adopted the approach of \cite{Titsias}, once again implemented in \texttt{gpytorch}. We also wrapped the ICM model in this way, and selected an identical number of inducing points $N_{ind}=500$ for all models. Tests were run for $q$ ranging from 1 to 7, and the best accuracy was reached for $q=7$ for all models. Results are displayed in table \ref{tab:sarcos} below. PLMCs can be seen to yield very satisfactory results (yet less good than these for the ICM), contrary to the variational model which struggles in this case, and with a very competitive training duration. Note that contrary to the previous experiments (which protocol is described in appendix \ref{an:exp_spec}), the training budget was here exhausted, so that all models were trained in $10^{5}$ iterations; the training loss however approximately stabilized in all cases. Observed differences in accuracy can therefore not be attributed to incomplete convergence; this is true in particular for the severe underperformance of the variational model, verified over many repetitions of the experiment with different training hyperparameters, and which causes remain unclear.\\
A last interesting observation is the poor values of the PVA for the OILMM and PLMC. This ill-calibration of the predictive variance should have different causes in each case: probably a too-rigid model in the first case (the only difference between PLMC-fast and OILMM is the constrained mixing matrix of the latter), and on the contrary a too-flexible, hard-to-optimize noise model for the unrestricted PLMC. 

\paragraph{Neutronics dataset}
This dataset consists in so-called \emph{homogenized cross-sections} (HXS), pre-processed nuclear data used in reactor core simulations -- see e.g \cite{ma_these}. The objective is one of surrogate modeling: rather than generating HXS's on-the-fly, reactor core simulators interpolate their tabulated values. The dataset contains 4 features corresponding to reactor operation parameters (fuel temperature, depletion of fissile material...), 201 input points and about $1.2 \cdot 10^{5}$ outputs (HXS's): the purpose of this experiment is thus to test our model on a dataset with very many tasks. Due to this very large value of $p$, only three models are tractable (see table \ref{tab:comp_comp} in \ref{an:complexities} for computational complexities) and thus compared: PLMC-fast, OILMM and the variational LMC. The data was classically standardized. Models are defined as before, with $\lfloor n / 1.5\rfloor $ inducing points for the \texttt{var} model (see section \ref{sec:test_models}), and 16 latent functions for each, a number found to be near-optimal for all in preliminary experiments. In this experiment too, the training budget $N_{iter}=50,000$ was exhausted by all models. Results are displayed in table \ref{tab:neutro}. These very low RMSE values obtained on unit-sized outputs are not abnormal: the data is of small effective linear dimension (i.e outputs are strongly correlated), virtually noiseless, and very smooth. Once again, results are favorable to the PLMC-fast, which is faster and more accurate than the variational baseline -- and also more accurate than the OILMM. More importantly, this experiment proves on real-life data that this model can be trained on datasets containing hundreds of thousands of tasks in a matter of minutes.\\

\begin{table}[!ht]
	\centering
	\captionsetup{justification=centering}
    \caption{Experimental results on dataset SARCOS ($N_{iter} = 100,000$).\\
    This dataset is high-dimensional ($d=21$) and contains $4,500$ datapoints.}\label{tab:sarcos}
    \begin{tabular}{|l||l|l|l|l|l|l|}
    \hline
        \textbf{Model} &\textbf{$t_{train}$} &\textbf{$R2$} &\textbf{$RMSE$} &\textbf{$Q^{95}_{L1}$} &\textbf{$PVA$} \\ \hline
        PLMC & 3268 & 0.987 & 0.116 & 0.239 & \textcolor{red}{5.96} \\
        PLMC_fast & \textcolor{green}{3164} & 0.987 & 0.115 & 0.234 & 0.36 \\
        oilmm & 3611 & 0.984 & 0.128 & 0.262 & -2.182 \\ \hline
        var & 3913 & \textcolor{red}{0.594} & \textcolor{red}{0.638} & \textcolor{red}{1.376} & \textcolor{green}{-0.218} \\
        ICM & \textcolor{red}{61,712} & \textcolor{green}{0.997} & \textcolor{green}{0.053} & \textcolor{green}{0.11} & -0.543 \\ \hline
        \bottomrule
    \end{tabular}

    \centering
    \quad
    \captionsetup{justification=centering}
    \caption{Experimental results on the neutronics dataset ($N_{iter} = 50,000$).\\
    This dataset contains very many tasks ($p \simeq 1.2 \cdot 10^{5}$).}\label{tab:neutro}
    \begin{tabular}{|l||l|l|l|l|l|}
    \hline
        \textbf{Model} &\textbf{$t_{train}$} &\textbf{$R2$} &\textbf{$RMSE$} &\textbf{$PVA$} \\ \hline
        PLMC-fast & 998 & \textcolor{green}{1.000} & \textcolor{green}{1.7e-2} & \textcolor{green}{-0.048}\\ 
        oilmm & \textcolor{green}{994} & 0.999 & 3.87e-2 & 0.355 \\ \hline
        var & \textcolor{red}{1402} & \textcolor{red}{0.998} & \textcolor{red}{4.24e-2} & \textcolor{red}{0.9229}\\
        \bottomrule
    \end{tabular}
\end{table}

\paragraph{Bramblemet tidal height dataset}
This dataset consists in weather data collected by four weather stations located in Southampton, UK \footnote{It's available for instance at \url{https://github.com/GAMES-UChile/mogptk}, or on the websites of said stations: \url{http://www.bramblemet.co.uk}, \url{http://www.cambermet.co.uk}, \url{http://www.chiemet.co.uk} and \url{http://www.sotonmet.co.uk}.}. We focus on the tidal height data amongst all records, thus obtaining 4 outputs (which we didn't normalize) and only one input, time (time-series data). Two weeks of observations (the two first of June 2020) were retained; 3 of the 4 outputs channels were linearly interpolated to match the time frame of the Bramblemet station, taken as reference. We subsampled the data to the quarter of its frequency, for computational ease and to improve convergence. In order to correctly model this highly periodic signal, we selected a spectral mixture kernel (\cite{spectral_mixture}). The learning task was to predict a whole day of measurements, put aside of the training data; it proved to be very challenging for all models, which tended to get stuck in local minima -- in particular this of a prediction worth zero everywhere -- despite proper initialization of parameters. We therefore tried all combinations of the number of latent processes $q$ and number of spectral components $N_{mix}$, in the ranges $\llbracket 1;4 \rrbracket$ and $\llbracket 2;5 \rrbracket$ respectively, and selected the best outcome for each model. Results are displayed in table \ref{tab:bramb}, with an illustration in figure \ref{fig:bramblemet}: we can see that only PLMCs managed to converge to a good approximation -- excluding OILMM, which performed very poorly. Once more, the PLMC-fast was up to the task, and achieved the best predictive performance despite its additional simplifications. This experiment proves that PLMCs can accommodate complex kernels to perform tasks such as time series forecasting, better even that competing models.  

\begin{figure}[hbtp]
    \centering
    \begin{subfigure}[hbtp]{\textwidth}
        \centering
        \includegraphics[width=0.9\textwidth]{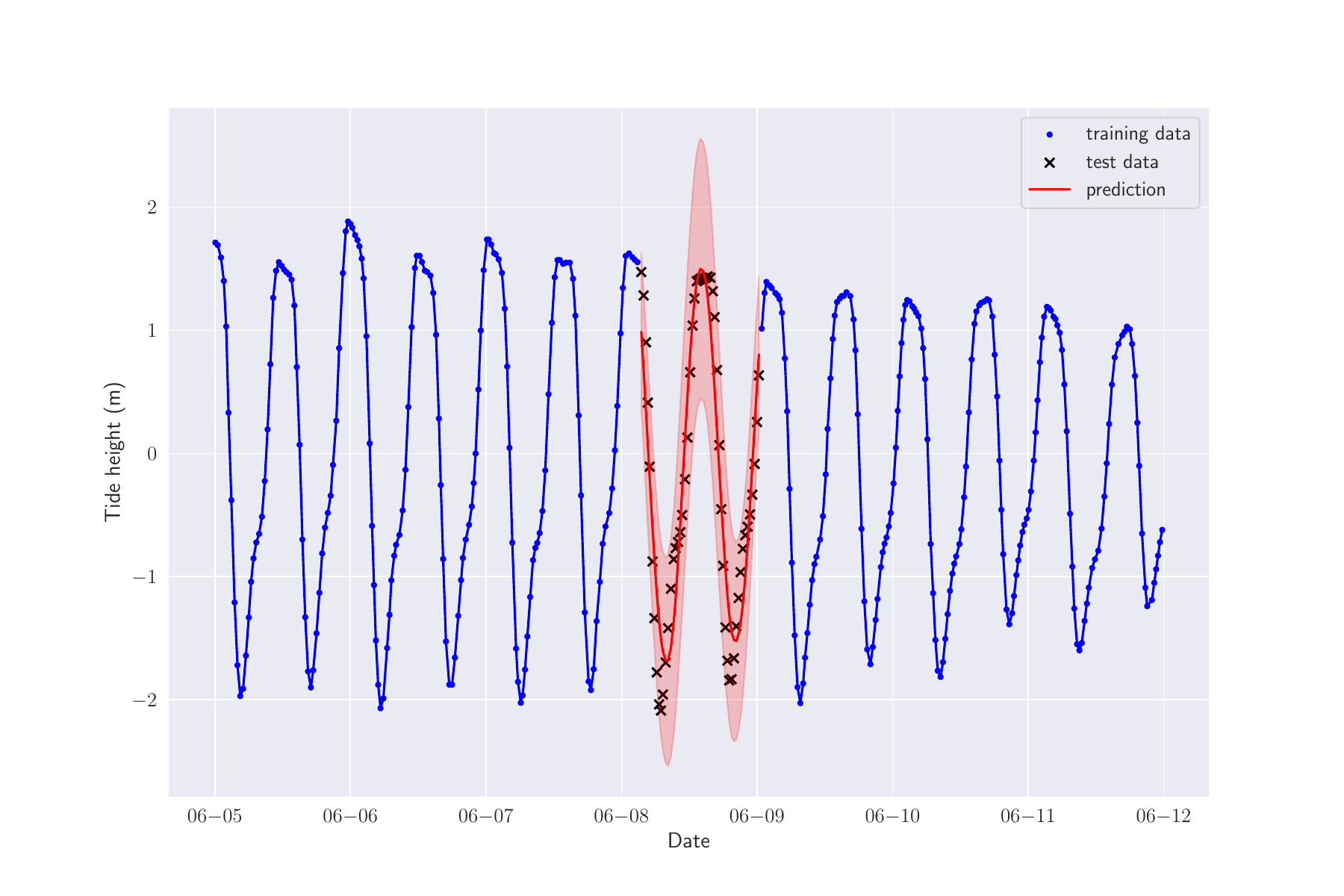}
        \caption{Bramblemet station}
    \end{subfigure}
    \begin{subfigure}[hbtp]{0.9\textwidth}
        \centering
        \includegraphics[width=\textwidth]{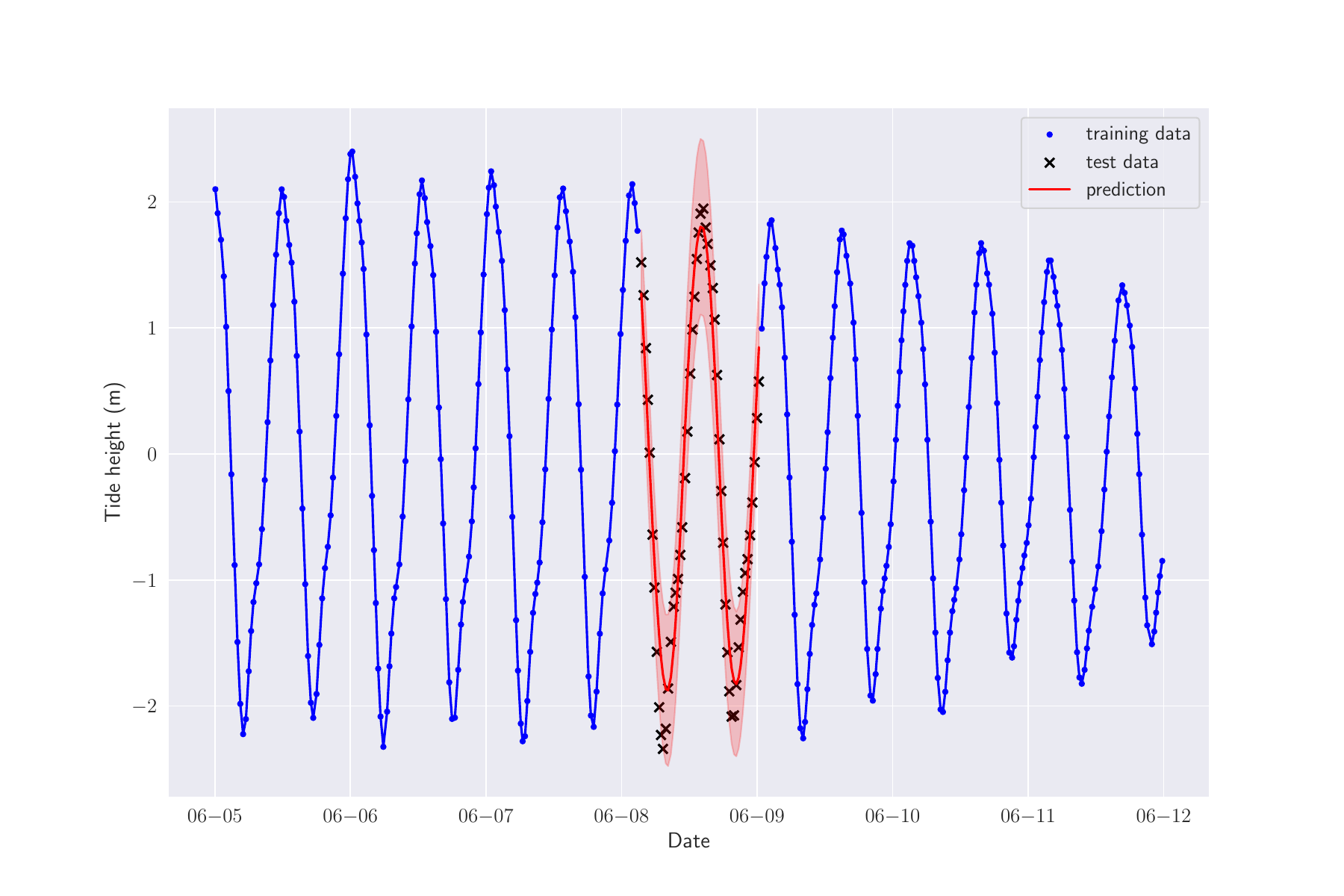}
        \caption{Cambermet station}
    \end{subfigure}
    \caption{Illustration of predictions of the PLMC-fast model on the tidal height dataset.}
    \label{fig:bramblemet}
\end{figure}

\begin{table}[hbtp]
	\centering
	\captionsetup{justification=centering}
	\caption{Results of the \emph{Bramblemet} experiment\\
	 ($p=4$, $n=156$, $d=1$)}
    \begin{tabular}{|l||l|l|l|l|l|l||l|l|}
    \hline
        \textbf{Model} & \textbf{$N_{iter}$} &\textbf{$t_{train}$} & $R2$ &\textbf{$RMSE$} & \textbf{$Q_{L1}^{95}$}&\textbf{$PVA$} &\textbf{$q$} &\textbf{$N_{mix}$}\\ \hline
        PLMC & 781 & \textcolor{green}{25} & 0.914 & 0.365 & \textcolor{green}{0.643} & 0.42 & 1 & 2 \\
        PLMC-fast & \textcolor{red}{2,652} & \textcolor{red}{292} & \textcolor{green}{0.921} & \textcolor{green}{0.348} & 0.652 & \textcolor{red}{-2.17} & 4 & 5 \\
        oilmm & 928 & 74 & 0.068 & 1.200 & \textcolor{red}{2.08} & 0.29 & 3 & 4 \\ \hline
        ICM & 2,255 & 27 & 0.103 & \textcolor{red}{1.271} & 1.99 & \textcolor{green}{0.23} & 1 & 2 \\ 
        var & \textcolor{green}{500} & 277 & \textcolor{red}{-0.030} & 1.187 & \textcolor{red}{2.08} & 0.67 & 1 & 5 \\
        \bottomrule 
    \end{tabular}
    \label{tab:bramb}
\end{table}

\section{Discussion: practical benefits and shortcomings of the PLMC and PLMC-fast}

\subsection{Conceptual simplicity and ease of implementation}
Simplicity is a strong suit of the PLMC-fast. As definition \ref{def:PLMC_fast} and figure \ref{fig:matrix_diag} show it, the model can be computed and understood with only three operations: project the data onto the latent space using the pseudoinverse of $\mathbf{H}$; perform independent GP computations with the latent processes; lift the results to the original space by multiplying them with $\mathbf{H}$. The only conceptual difficulty is the necessity of a complex noise model (DPN condition). The fact that the core of all computations is just standard GP manipulations on latent processes means that \textbf{the model can reuse most GP machinery developed for single-output GPs} -- approximations, mini-batch training, computation of specific quantities... -- at the expense of minor adaptations.\\
By contrast, an informed use of variational LMCs requires delving into an elaborate set of approximations -- see for instance \cite{gpflow}. Implementing such models is also more complex, and cannot be built upon standard GP implementations. As for the ICM, although its definition is straightforward, its efficient implementations such as this of \cite{all_in_the_noise} are much less, and significantly depart from standard GP computations: they involve clever factorizations, Kronecker product operations, eigen-decompositions, etc. This means that adding any of the above-mentioned functionalities to it requires mathematical and software work to adapt it to this specific framework.

\subsection{Training data modifications}
With the PLMC, a practitioner can access the projected data $\mathbf{TY}$, and latent-level posteriors and MLL of the model. This is also the case with variational models, where the variational distribution plays the role of this low-dimensional data (see \ref{an:comparison}); but this distribution is a black-box, parametrized quantity with no explicit relation to the training data, whereas in the PLMC it is derived from it through a simple projection. The ICM, for its part, offers no equivalent functionality: its latent posteriors could theoretically be computed using eq.\eqref{eq:estimators}, but this would require an additional implementation (with prior mathematical work to make it efficient) unrelated to the internal machinery of the model.\\
Such access makes it possible to adapt data modification operations such as training set updates, crucial in the contexts of active learning and bayesian optimization, to the case of multi-task processes. To add new data to the PLMC, one must simply:

\begin{itemize}
    \item Project the new output vector $\mathbf{Y_{n+1}}$ into the latent space with $\mathbf{T}$;
    \item Stack the resulting vector with the existing projected data;
    \item Update all latent kernel matrices $\mathbf{K_{i}}$ with rank-1 update formulas, or equivalent constructions.
\end{itemize}
Regarding the last point, kernel matrix updates are trivial to perform with exact single-output GPs: noting $\mathbf{K_{n}}$ a kernel matrix for the datapoints up to $\mathbf{x_{n}}$, the inverse matrix with an added point $\mathbf{x_{n+1}}$ is simply \cite{rank1_updates}:

\begin{equation}
    \mathbf{K_{n+1}}^{-1} = \left[ \begin{array}{c c} \mathbf{K_{n}}^{-1} & -c \mathbf{v} \\
    -c \mathbf{v^{T}} & c\end{array}\right]
\end{equation}
with $\mathbf{v}=\mathbf{K_{n}}^{-1} k(\mathbf{X_{n}, x_{n+1}})$ and $c=(k(\mathbf{x_{n+1}, x_{n+1}}) - \mathbf{v}^{T}k(\mathbf{X_{n}, x_{n+1}}))^{-1}$. Similar expressions exist for removing a point from the data, or for adding several points simultaneously \cite{rank1_updates}; and similar updates exist for different computation schemes, for instance when $\mathbf{K_{n}}$ is handled through its Cholesky decomposition \cite{osborne_thesis} or in the LOVE framework \cite{LOVE_BO}.\\

On the other hand, to our best knowledge, no methodology has yet been proposed for incorporating new data to a multitask variational GP; this is mentioned as a research direction by a recent influential article addressing the single-output case \cite{online_conditioning}. What is lacking to achieve this generalization seems to be precisely the absence of a mechanism for converting new tasks-level observations into latent-level ones. Moreover, training set updates -- often framed as "online conditioning" in this context -- are very intricate to perform even in the single-output case, and constitute a topic of research in itself, as the above-mentioned article illustrates it\footnote{In \cite{online_conditioning}, adding new data to a variational GP requires converting its black-box parameters to a pseudo-data representation, building a new GP out of the new data and pseudo-data, retrieving sensible values of the black-box parameters from this GP, and updating inducing points in a sensible manner. Several of these steps involve heuristics, and therefore possible accuracy losses.}.\\

Another useful data manipulation is leave-one-out cross-validation: with exact single-output GPs, LOO estimates can be obtained from efficient closed-form expressions given for instance in chapter 5 of \cite{Rasmussen}. These expressions directly translate to the PLMC thanks to the linear relation between latent and task-level data; calling $\mathbf{u}^{LOO} $ the leave-one-out errors of the model and $\vartheta^{LOO}$ its leave-one-out predictive variance, it stands:

\begin{align}\label{eq:PLMC_LOO}
	\left[ \mathbf{TY} - \mathbf{u}^{LOO}\right]_{ij} &= \left[ (\mathbf{K_{j}} + \sigma_{j} \mathbf{I_{n}})^{-1}\mathbf{YT_{j}^{T}} \right]_{i} / \vartheta_{ij}^{LOO} \quad \forall j \in \left[ 1;q \right]\\
	\vartheta_{ij}^{LOO} &= \left[ (\mathbf{K_{j}} + \sigma_{j} \mathbf{I_{n}})^{-1} \right]_{ii} \quad \forall j \in \left[ 1;q \right]\\
	\mathbf{Y} - \mathbf{Y^{LOO}} &= (\mathbf{I_{p} - HT})\, \mathbf{Y} + \mathbf{H}\, (\mathbf{TY - u^{LOO}})\\
	\mathbf{v^{LOO}} &= \mathbf{H} \vartheta^{LOO} \mathbf{H^{T}}\label{eq:PLMC_LOO_end}
\end{align}

Meanwhile, no such expressions can exist for variational models, which black-box parameters cannot forget the information brought by a datapoint.

\subsection{Model and data interpretation}
The latent quantities provided by the model can yield valuable insights about the data. The projected data $\mathbf{TY}$, for instance, is analogous to the principal components (PCA) of $\mathbf{Y}$, with a notable difference: here, the values of the $i$-th latent component $\left[ \mathbf{TY}\right]_{i}$ are correlated by a gaussian process, whereas they are considered independent in a (probabilistic) PCA. These latent factors are thus expected to be smoother and more physically meaningful than principal components; they could be used to identify significant variation modes in climate data, physics simulations... This was the purpose of the already-cited Generalized Probabilistic PCA \cite{GPPCA}, which the PLMC generalizes. It is also noteworthy that the model explicitly splits the data into two constituents: $\mathbf{HTY}$, which it uses\footnote{$\mathbf{TY}$ is a sufficient statistic of the data for the latent processes, see proposition \ref{prop:interpretation}.}, and $\mathbf{(I_{p} - HT)Y}$, which it discards.

\subsection{Scalability with the number of tasks}
Computational costs of the PLMC-fast scale less than linearly with the number of outputs $p$ (see appendix \ref{an:complexities} for more details), which authorizes to use it in datasets with very many tasks. This is unlike the ICM, for which even optimized implementations involve computations with complexity $\mathcal{O}(p^{3})$ (see appendix \ref{an:complexities} and \cite{all_in_the_noise}). Low-rank versions of the ICM can be devised to break this cubic complexity -- this is done for instance in \cite{MTGP} -- but this once again requires dedicated mathematical and software optimizations, which may not be compatible with other computational devices (approximations, computation of entropy, etc). Variational LMCs, on the other hand, also scale linearly with $p$, but some of their features prevent them from being used in some applications -- see previous paragraphs.

\subsection{Limitations}
A first limitation to mention is that the PLMC doesn't natively fix the cubic cost in the number of datapoints of exact GPs, contrary to variational models. Of course, the PLMC can be easily wrapped with approximations to lift this bottleneck, as it was done in several experiments of this article; but such approximations may negate some of the benefits of the model, such as the possibility to modify the training data.\\
Another issue is the constraint imposed on the noise covariance: this prevents using the PLMC for uncertainty propagation, for in this setting the upstream uncertainty covariance wouldn't respect the DPN condition and couldn't be injected into the model. To solve this problem, one could derive an alternative parametrization of the model, maintaining a fully general noise covariance $\mathbf{\Sigma}$ and displacing the DPN constraint on the mixing matrix $\mathbf{H}$. This mathematical work has been undertaken and will be the subject of future publication.

\subsection{Ethical considerations}
The application domains of MOGPs include sensitive fields (healthcare, nuclear engineering, robotics, etc). In some use cases, computing faithful uncertainty estimates may be crucial for decision-making. The PLMC having a restricted noise model, it should not be used in such applications, unless the implications of this restriction are perfectly understood by the practitioner in his or her specific setting, and the consequences proven acceptable. 

\section{Conclusion}
In this article, we showed that the equations of the Linear Model of Co-regionalization could be reformulated in a way that reveals its low-rank latent structure. The new expressions have the striking property of decoupling if and only if a certain condition on the noise -- already formulated in \cite{OILMM} -- is verified, enabling computations to be linear in the number of latent processes rather than cubic in the number of tasks. We gave a general parametrization of LMC models enforcing this noise assumption, thus bringing out a new model called \emph{Projected LMC} (PLMC), as well as a simplified and more efficient version of it (PLMC-fast). We tested these models on real and synthetic data against concurrent approaches, and studied their behavior as data parameters vary. PLMCs showed themselves to compete very well with state-of-the-art variational LMC and Intrinsic Co-regionalization Model, being faster and equally or more accurate, even in setups where they should theoretically be challenged (large and highly-structured noises).\\

In addition to yielding theoretical insights\footnote{See for instance \ref{an:comparison} for a comparison of the approximations made by several LMC models.}, the PLMC appears as a viable way to make the LMC framework scalable with the number of tasks, and a simpler alternative to its variational and Kronecker-covariance counterparts. It can be implemented as a mere overlay of a batch single-output GPs, and reuse all computational machinery (approximations, computation of probabilistic quantities, etc) devised for the latter. Moreover, contrary to its competitors, it preserves an explicit relation -- which is fortunately linear -- between the observed data and its low-dimensional summary used in computations. This greatly facilitates data-related operations such as incorporation of new observations, which are not even available as of now for multitask variational models, making methodologies like multitask bayesian optimization and multitask active learning accessible to much wider research and engineering communities.  

% References
% \bibliographystyle{unsrt}
% \bibliography{biblio}

\newpage

\appendix

\section{Table of symbols}

\begin{itemize}
	\item $n$ : number of training datapoints
	\item $p$ : number of modeled tasks
	\item $q$ : number or latent processes
	\item $d$ : dimension of the inputs (number of variables)
 	\item $\mathbf{H}$ : mixing matrix of the LMC model, of size $p \times q$
 	\item $\mathbf{T}$ : projection matrix, pseudoinverse of $\mathbf{H}$ of size $q \times p$. Expression given in definition \ref{def:T_sigma}
 	\item $\mathbf{Q, R}$ : QR decomposition of $\mathbf{H}$
 	\item $\mathbf{Q_{\bot}}$ orthonormal complement of $\mathbf{H}$
	\item $\bm{\epsilon}$ : noise process, correlated across all tasks
	\item $\mathbf{\Sigma}$ : cross-tasks noise (covariance of $\bm{\epsilon}$), of size $p\times p$
	\item $\mathbf{\Sigma_{P}}$ : projected noise, i.e $q\times q$ covariance of the noise process $\mathbf{T}\bm{\epsilon}$
	\item $\mathbf{\Sigma_{\bot}}$ : $(p-q)\times (p-q)$ covariance of the discarded noise (noise projected onto $Span(H)^{\bot}$)
	\item $\tilde{\mathbf{\Sigma}}_{\bot}$ : blockwise inverse of $\mathbf{\Sigma_{\bot}}$ inside the matrix $\mathbf{D_{+}}$ (see definition \ref{cor:facto_Q})
	\item $\sigma_{\bot}$ : discarded noise of the PLMC-fast. For this model, $\mathbf{\Sigma_{\bot}} = \sigma_{\bot} \mathbf{I_{p-q}}$ 
	\item $\mathbf{M}$ : noise term coupling $\mathbf{\Sigma_{P}}$ and $\mathbf{\Sigma_{\bot}}$
	\item $\mathbf{X}$ : matrix containing all training data inputs, of size $d \times n$
	\item $\mathbf{Y}$ : matrix containing all training data outputs, of size $p \times n$
	\item $\mathbf{U}$ : $q \times n$ matrix of values of the latent processes at the observation points $\mathbf{X}$
	\item $x_{*}$ : input point at which a prediction is made
	\item $k_{i}$ : kernel function of the $i$-th latent process
	\item $\mathbf{k_{i*}}$ : vector $k_{i}(\mathbf{X, x_{*}})$ ($=[k(\mathbf{x_{j}, x_{*}})]_{1\leq j \leq n}$)
	\item $\mathbf{K_{i}}$ : matrix $k_{i}(\mathbf{X, X})$ ($=[k(\mathbf{x_{j}, x_{l}})]_{1\leq j \leq n}^{1\leq l \leq n}$)
	\item $k_{i**}$ : scalar $k_{i}(\mathbf{x_{*}, x_{*}})$
\end{itemize}

\section{Code}

Model implementation and experiments are available at \url{https://github.com/QWERTY6191/projected-lmc} .

\section{Computational complexities}\label{an:complexities}

We here compare the computational complexities, both in time and memory space, of inference and MLL computation for all models considered in this article. References for the stated values are given in the last column of table \ref{tab:comp_comp}. For inference, we make the distinction between computation of caches (inverse kernel matrices, projection matrix...) which values are independent of test points and can thus be precomputed, and test-dependent values themselves. We also denote by $m$ the number of inducing points for the variational model; we recall that the variational approach considered here is this of \cite{hensman_class}.

\begin{table}[!ht]
    \centering
    \caption{Computational complexities of several models.}\label{tab:comp_comp}
    \resizebox{\columnwidth}{!}{%
    \begin{tabular}{|l||l|l||l|l||l|l||l|}
    \hline
        \textbf{Model} & \makecell{Time \\ (inference)} & \makecell{Memory\\ (inference)} & \makecell{Time\\ (inf. caches)} & \makecell{Memory\\ (inf. caches)} & \makecell{Time\\ (MLL)} & \makecell{Memory\\ (MLL)} & References \\ \hline
		\makecell{PLMC} & $O(pq + qn)$ & $O(pq + qn)$ & $O(qn^{3} + pqn + p^{3})$ & $O(\max(qn^{2}, p^{2}))$ & $O(qn^{3} + pqn + p^{3})$ & $O(qn^{2} + p^{2})$ & - \\ \hline
		\makecell{PLMC-fast} & $O(pq + qn)$ & $O(pq + qn)$ & $O(qn^{3} + pqn + pq^{2})$ & $O(\max(qn^{2}, pq))$ & $O(qn^{3} + pqn + pq^{2})$ & $O(qn^{2} + pq)$ & - \\ \hline
		\makecell{OILMM} & $O(pq + qn)$ & $O(pq + qn)$ & $O(qn^{3} + pqn + pq^{2})$ & $O(\max(qn^{2}, pq))$ & $O(qn^{3} + pqn + pq^{2})$ & $O(qn^{2} + pq)$ & - \\ \hline
		\makecell{ICM} & $O(pn)$ & $O(pn)$ & $O(n^{3} + p^{3})$ & $O(n^{2} + p^{2})$ & $O(n^{3} + p^{3})$ & $O(n^{2} + p^{2})$ & \cite{all_in_the_noise} \\ \hline
		\makecell{Variational} & $O(pq + qm)$ & $O(pq + qm)$ & $O(qm^{3})$ & $O(qm^{2})$ & $O(qnm^{2} + pqn)$ & $O(qm^{2})$ & \makecell{\cite{hensman_class}, \\ \cite{generic_inference}}\ \\ \hline
    \end{tabular}%
    }
\end{table}

The main differences between models are:

\begin{itemize}
\item The effect of the inducing points approximation for the variational model: all kernel matrices of size $n\times n$ are replaced by inducing points matrices of size $m\times m$, replacing all complexity terms in $n^{k}$ by terms in $m^{k}$. Note that the same gains are achievable when coupling PLMCs with inducing points approximations.
\item The extra term $O(pqn)$ in the cache computation of the PLMCs, which represents data projection (computation of $\mathbf{TY}$).
\item The term $O(pqn)$ in the MLL duration of the variational model, which comes from the computation of likelihood terms in the observation space, with no calculation trick.
\item The fact that the ICM comprises terms in $\mathcal{O}(n^{3})$ rather than $\mathcal{O}(qn^{3})$: this comes from the latent processes sharing the same kernel in this model.
\item The fact that the general PLMC and the ICM with general $\mathbf{K^{f}}$ feature terms in $p^{3}$ rather than $pq^{2}$, showing that they don't fully leverage the low-dimensional structure of the LMC.\\
\end{itemize}

\section{Experimental specifications}\label{an:exp_spec}

For each model, the training dynamics were investigated individually on the synthetic data in order to verify that performance discrepancies weren't due to improper convergence. A list of optimization options were tried out; we found a standard set of items to perform well for all models, yielding well-behaved and mostly monotonous training in all cases. We therefore used these common hyperparameters for all models and experiments:

\begin{itemize}
 \item The optimizer is the \texttt{AdamW} algorithm of the \texttt{torch.optim} package (\cite{torch}, with default hyperparameters.
 \item We adopted a linear decay scheduler: the global learning rate (parameter $\gamma$) of AdamW decreases linearly from a maximal to a minimal learning rate in $N_{lin}=10^{4}$ iterations.
 \item The maximal learning rate is set to $10^{-2}$, and the minimal to $10^{-3}$.
 \item A stopping criterion was adopted: training stops when differences between the average over $N_{patience}=500$ iterations of relative loss differences ($1 - \mathcal{L}_{i+1} / \mathcal{L}_{i}$) became smaller than a given threshold $\delta \mathcal{L}_{thresh}$. For real-data experiments,  $\delta \mathcal{L}_{thresh} = 10^{-7}$, while it was set to $\delta \mathcal{L} = 2.5 \cdot 10^{-6}$ for synthetic data.
 \item The maximal number of iterations $N_{iter}^{max}$ was set to $10^{5}$ in some experiments and $5 \cdot 10^{4}$ in others. In synthetic data experiments and the Bramblemet one, these values are unimportant as no model attained the limit. On the contrary, in 3 out of the 4 real-data experiments, the budget was exhausted by all models; in this case, it is clearly indicated with the presentation of results.
 \item The LMC coefficients of the PLMC were initialized by setting $\mathbf{Q_{+}^{start}=V}$, $\mathbf{R_{start}=S_{\left[ 1:q \right]}}$ with $\mathbf{Y=USV^{T}}$ the SVD of the data. Likewise, the mixing matrix of the ICM\footnote{That is, the low-rank factor of its cross-tasks covariance matrix.} and this of the variational LMC was initialized to $\mathbf{SV^{T}}_{\left[ 1:q \right]}$. Initialization of other model parameters is straightforward: all characteristic quantities of the toy data being of unit order, setting kernel parameters to $\sim 1$ and noise parameters to $\sim  10^{-2}$ always yields satisfactory results.
\end{itemize}

Regarding hardware specifications, all synthetic data experiments were run on a Xeon Skylake 5218 2.3 GHz GPU with 12 Go of RAM, while real-data ones were performed on a 80-CPU (2.10GHz Intel(R) Xeon(R) Gold 6230) machine.

\section{Additional experiments}\label{an:extra_exp}
\subsection{Additional observations on the parametric study}

All models didn't respond in the same way to an increase in $p$ with fixed $q$, or conversely to an increase in $q$ with fixed $p$; such experiments are displayed in figure \ref{fig:p_q_var}. First, for the PLMC and PLMC-fast, there are signs of instabilities\footnote{Corroborated by a sharp increase in the training duration, not displayed here.} when $q$ approaches $p$; it may be that the optimization landscape becomes too flat, with many close configurations of latent processes yielding similar results. This would explain that the simpler, more regularized models perform best in this setting: OILMM bests PLMC-fast, which bests PLMC. A more surprising observation is that the \texttt{ICM} and PLMC are sensitive to the \emph{absolute number of tasks}. For the ICM this result is perhaps not surprising, as its kernel structure is unsuited to our synthetic data, which contains as many lengthscales as there are latent processes\footnote{Yet the range of these lengthscales is kept constant, whatever $p$ and $q$ may be.} (see data description in section \ref{sec:synth}). For the full PLMC, it may rather be a problem of optimization: because of its many parameters, the optimization landscape may become too intricate in many-tasks complex problems.\\
Otherwise, varying the other data parameters (modifying the number of training points, reducing the range of kernel lengthscales...) offered very little discriminative power over all models.\\

\begin{figure}[hbtp]
    \centering
    \begin{subfigure}[hbtp]{0.6\paperwidth}
        \centering
        \includegraphics[width=\textwidth]{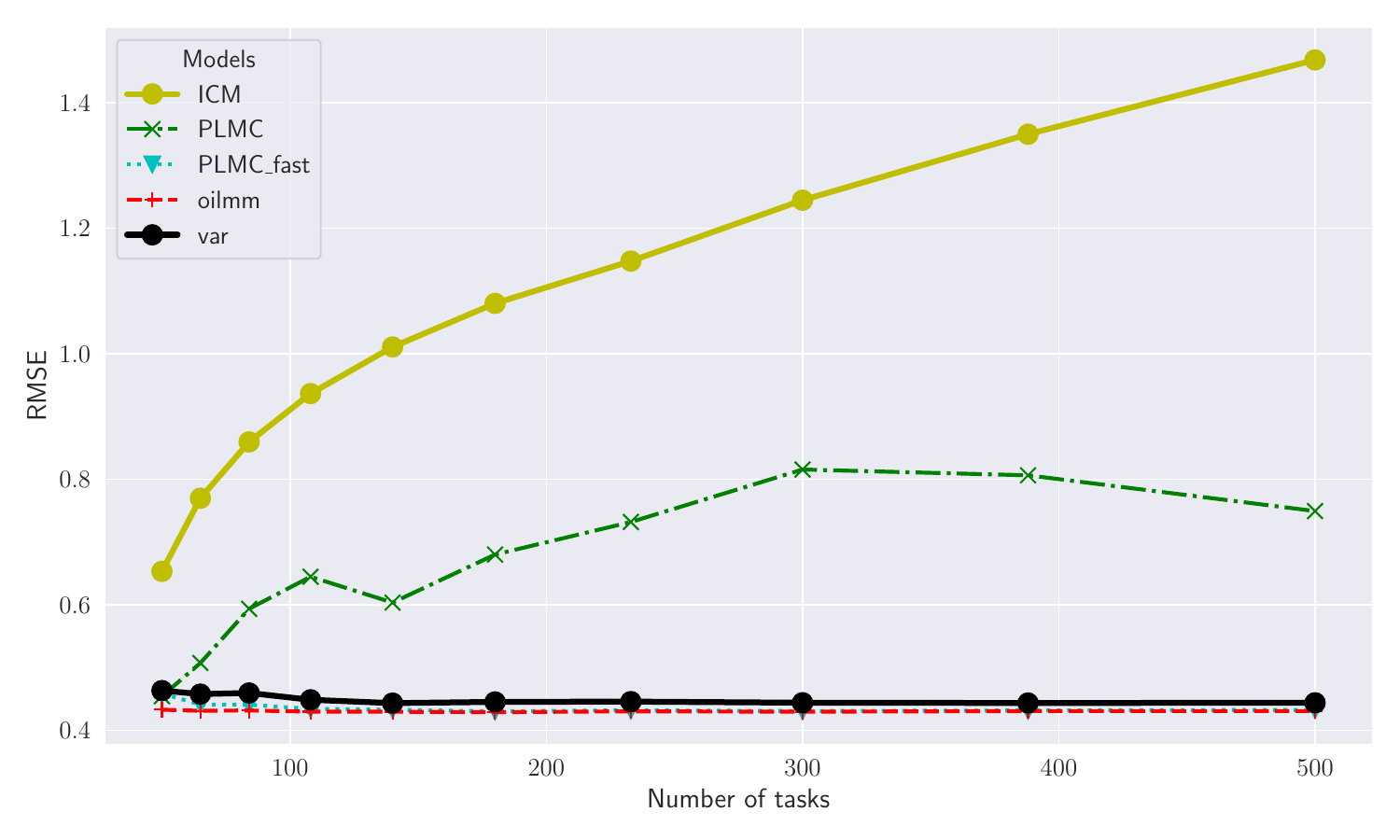}
        \caption{Variable $p$, fixed $q=25$}
        \label{fig:RMSE_ntasks}
    \end{subfigure}
    \quad
    \begin{subfigure}[hbtp]{0.6\paperwidth}
        \centering
        \includegraphics[width=\textwidth]{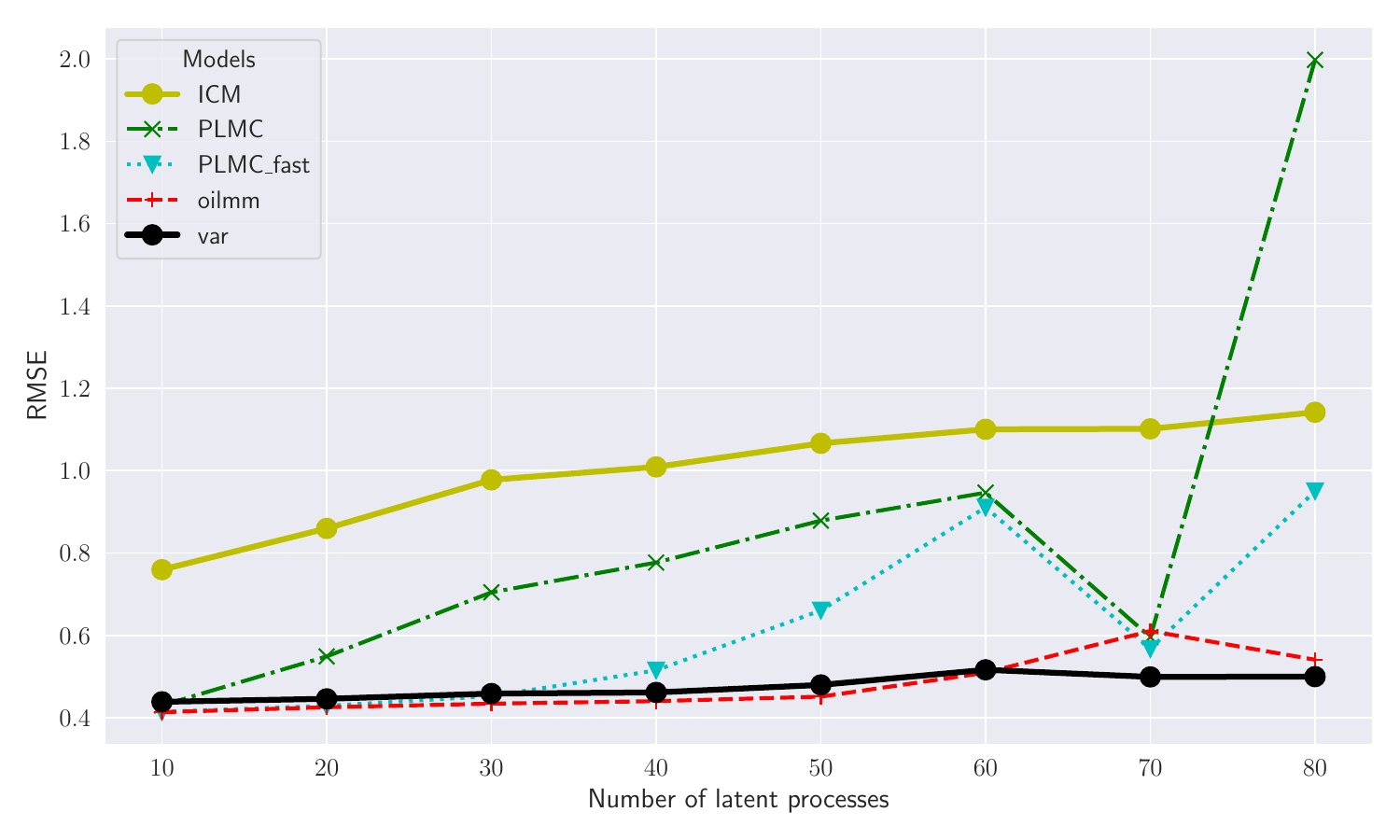}
        \caption{Variable $q$, fixed $p=100$}
        \label{fig:RMSE_nlat}
    \end{subfigure}
    \caption{RMSE of several models, with a) increasing number of tasks and fixed number of latent processes , and b) the opposite. Default noise parameters ($\mu_{noise}=0.1$, $\mu_{str}=0.9$, $q_{noise}=25$). Averaged over $N_{rep}=50$ random datasets}
    \label{fig:p_q_var}
\end{figure}

\subsection{Ship maintenance dataset}

This dataset was introduced in \cite{ship}, and is available at the UCI machine learning repository \footnote{\url{http://archive.ics.uci.edu/dataset/316/condition+based+maintenance+of+naval+propulsion+plants}}. It consists in numerical simulations of the propulsion system of a naval vessel, fine-tuned on real data. Inputs of the dataset are ship speed and two simulation coefficients reflecting the condition of the engine. These three inputs are sampled on a uniform grid with a good granularity. Outputs are 14 variables describing the operational state of the system: torques, pressures, fuel flow... We discarded two of these outputs presenting zero or near-zero variance, and standardized the others. For computational convenience, we subsampled the fine input grid by a factor of 5, and stored apart the last 100 points of the dataset for testing, ending up with 2287 training points. We once again used a Matérn-5/2 kernel with input-specific parametrized lengthscales for all models, and independent task noises (diagonal $\mathbf{\Sigma}$) for models \texttt{var} and \texttt{ICM}. Results are diplayed in table \ref{tab:ship}: once more, the PLMC-fast is faster than all models (it converges faster than the OILMM in particular). It achieves an accuracy very similar to this of the OILMM and variational model; all of them are however bested by the ICM. However, the full PLMC does not seem to have fully converged at the end of the training budget ($50,000$ iterations), resulting in subpar performance. Its larger number of parameters may have hindered its fitting to such high-dimensional and very smooth data.

\begin{table}[!ht]
    \centering
    \captionsetup{justification=centering}
	\caption{Results of the \emph{Ship maintenance} experiment\\
 ($p=12$, $n=2287$, $d=3$, $q=3$, $N_{ind}=500$)}
    \begin{tabular}{|l||l|l|l|l|l|l|}
    \hline
         \textbf{Model}& \textbf{$N_{iter}$} &\textbf{$t_{train}$} & $R2$ &\textbf{$RMSE$} &\textbf{$Q^{95}_{L1}$} &\textbf{$PVA$} \\ \hline
		\texttt{PLMC} & \textcolor{red}{50,000} & 1052 & \textcolor{red}{0.990} & \textcolor{red}{9.85e-2} & \textcolor{red}{1.96e-1} & -0.49 \\
		\texttt{PLMC-fast} & \textcolor{green}{7,060} & \textcolor{green}{139} & 0.994 & 7.57e-2 & 1.34e-1 & -0.66 \\
		\hline \texttt{oilmm} & 13,490 & 340 & 0.994 & 7.76e-2 & 1.63e-1 & -0.63 \\
		\texttt{var} & 14,541 & 248 & 0.994 & 7.45e-2 & 1.60e-1 & \textcolor{green}{-0.24} \\
		\texttt{ICM} & 14,213 & \textcolor{red}{1904} & 0.\textcolor{green}{997} & \textcolor{green}{5.42e-2} & \textcolor{green}{7.42e-2} & \textcolor{red}{-0.74} \\
        \bottomrule
    \end{tabular}
    \label{tab:ship}
\end{table}

\newpage
\section{Proofs}\label{an:proofs}

\subsection{Preliminary considerations}

\paragraph{Notation} We adopt the following standard notation for any matrix $\mathbf{M}$: $\quad \mathbf{Sym(M) = \frac{M + M^{T}}{2}}$. 
\paragraph{Kronecker product manipulation} \textit{We recall elementary properties of the Kronecker product}:
\begin{align}
\mathbf{(A \otimes B)(C \otimes D)} &= \mathbf{(AC \otimes BD)} \quad \text{for all matrices} \ \mathbf{A, B, C, D} \ \text{for which the products are defined;} \label{fact1}\\
\mathbf{(C^{T} \otimes A)vec(B)} &= \mathbf{vec(ABC)} \label{fact2} \\
\mathbf{(A \otimes B)^{-1}} &= \mathbf{A^{-1} \otimes B^{-1} \quad \text{(the inverse exists if and only if} \ A^{-1}, \, B^{-1} \ \text{exist)} } \label{fact3}
\end{align}

\paragraph{An identity on gaussian processes} 
\begin{lemma}\label{lem:cond_expect}\begin{equation}
\mathbb{E}_{U|Y} \left[ \mathbb{E}_{u^{*}} ( \mathbf{u^{*} | Y,U} ) \right] =  \mathbf{ Diag(k_{i*}^{T}) \, Diag(K_{i}^{-1})} \, \mathbb{E}_{U|Y}(\mathbf{U}_{v}|\mathbf{Y}) 
\end{equation}\end{lemma}
\begin{proof}\begin{align*}
\mathbb{E}_{U|Y} \left[ \mathbb{E}_{u^{*}} ( \mathbf{u^{*} | Y,U} ) \right] &= \int \mathbb{E}_{u^{*}} ( \mathbf{u^{*} | \cancel{Y},U} ) \, p(\mathbf{U|Y}) \, \mathbf{dU} \\
&= \int \mathbf{ Diag(k_{i*}^{T}) \, Diag(K_{i}^{-1}) \, U_{v} \, } p(\mathbf{U|Y}) \, \mathbf{dU} \ \text{} \\
&= \mathbf{ Diag(k_{i*}^{T}) \, Diag(K_{i}^{-1})} \, \mathbb{E}_{U|Y}(\mathbf{U}_{v}|\mathbf{Y}) 
\end{align*}
The second equality holds by the conditioning property of GPs, and because $ \mathbb{E}_{u^{*}} ( \mathbf{u^{*} | Y,U} ) $ doesn't depend on $\mathbf{Y}$ (conditionnally on their corresponding latent values, observed values of a GP are independent of all other variables of the model).
\end{proof}

\paragraph{A useful matrix identity}
\begin{lemma}\label{fact4}
For all invertible matrices $ \mathbf{A, B}$ and all matrix $ \mathbf{C}$ conformable with them,\\
 $\mathbf{C^{T}(CAC^{T} + B)^{-1}C = (A + C^{T}BC)^{-1}}$. Note that $ \mathbf{A}$ and $ \mathbf{B}$ are not necessarily of the same size, so $\mathbf{C}$ can be rectangular.
\end{lemma}
\begin{proof}
\begin{align*}
\mathbf{C^{T}(CAC^{T} + B)^{-1}C } &= \mathbf{ A^{-1} \left( \, A C^{T}(CAC^{T} + B)^{-1}C A \, \right) A^{-1} }\\
&= \mathbf{ A^{-1} \left( \, A C^{T}(CAC^{T} + B)^{-1}C A - A \, \right) A^{-1} \, + \, A^{-1} } \\
&= \mathbf{A^{-1} - \left( \, A^{-1} + C^{T}BC \right)^{-1} A^{-1} \, + \, A^{-1}} \quad \text{by backward Woodburry's identity} \\
&= \mathbf{(A + C^{T}BC)^{-1}} \quad \text{again by backward Woodburry's identity.}
\end{align*}
\end{proof}

\subsection{Proofs of the propositions}
\vspace{1cm}
\begin{lemma}\label{prop:post_U}
The posterior $p(\mathbf{U_{v}}|\mathbf{Y})$ of the latent processes $\mathbf{U}$ at the training points is gaussian with mean and variance:
\begin{align}\label{eq:post_U}
\mathbb{E}(\mathbf{U_{v}}|\mathbf{Y}) &= \mathbf{ \left[ Diag(K_{i}^{-1}) + H^{T}\Sigma^{-1}H \otimes I_{n} \right]^{-1} } \cdot \mathbf{ vec(Y^{T}\Sigma^{-1}H)} \\
\mathbb{V}(\mathbf{U_{v}}|\mathbf{Y}) &= \mathbf{ \left( Diag(K_{i}^{-1}) + H^{T}\Sigma^{-1}H \otimes I_{n} \right)^{-1} }
\end{align}\end{lemma}

\begin{rem}
It can be noted that \ref{eq:post_U} is only a blockwise version of the result (2.116) of Bishop's textbook (\cite{Bishop}) for the conditional law of a linear combination of joint gaussian variables, applied here with $\mathbf{x=U_{v}}$, $\mathbf{A = H}$, $\mathbf{L = \Sigma^{-1}}$, $\mathbf{  \Lambda = Diag(K_{i}^{-1})}$ and $\bm{\mu} = \mathbf{0, \ b=0}$. The only difference is that the vectors $\mathbf{x}$ and $\mathbf{y}$ in Bishop contain one value per "task", whereas our vectors $\mathbf{U_{v}}$ and $\mathbf{Y_{v}}$ contain gaussian subvectors of size $n$.\\
\end{rem}

\begin{proof}
All involved variables being jointly gaussian, we know that $p(\mathbf{U}|\mathbf{Y})$ is gaussian; it suffices to find its mean and variance. By Baye's rule, $p(\mathbf{U}|\mathbf{Y}) \propto p(\mathbf{Y}|\mathbf{U})\, p(\mathbf{U})$, the proportionality constant being the likelihood $p(\mathbf{Y})$ which is independent of the latent variables $\mathbf{U}$. By definition of the model, we have (using the symmetry of $\mathbf{\Sigma}$ in the second equality):
\begin{gather*}\label{eq:post_source}
p(\mathbf{Y}|\mathbf{U})\, p(\mathbf{U}) \propto \exp \lbrace -\frac{1}{2} \mathbf{ \left( Y_{v} - (H \otimes I_{n})U_{v} \right)^{T} (\Sigma^{-1}\otimes I_{n}) \left( Y_{v} - (H \otimes I_{n})U_{v} \right) } \ -\frac{1}{2} \mathbf{ U_{v}^{T} Diag(\mathbf{K}_{i}^{-1}) U_{v} } \rbrace \nonumber \\
\propto \exp \lbrace -\frac{1}{2} \mathbf{ Y_{v}^{T}(\Sigma^{-1}\otimes I_{n}) Y_{v} \ - \ 2 U_{v}^{T}(H^{T}\Sigma^{-1} \otimes I_{n})Y_{v} \ + \ U_{v}^{T}( Diag(\mathbf{K}_{i}^{-1}) + H^{T}\Sigma^{-1}H \otimes I_{n} ) U_{v} } \rbrace 
\end{gather*}
On the other hand, $p(\mathbf{Y}|\mathbf{U})\, p(\mathbf{U})$ being gaussian, there exists a covariance matrix $\mathbf{\tilde{K}}$ and a mean vector $\mathbf{\hat{U}_{v}}$ such that:
\begin{align*}
p(\mathbf{Y}|\mathbf{U})\, p(\mathbf{U}) & \propto \exp \lbrace -\frac{1}{2} \mathbf{ (U_{v} - \hat{U}_{v})^{T}\tilde{K}^{-1} (U_{v} - \hat{U}_{v}) } \rbrace \nonumber \\
&= \exp \lbrace -\frac{1}{2} \mathbf{ U_{v}^{T}\tilde{K}^{-1} U_{v} \ - \ 2 U_{v}^{T}\tilde{K}^{-1} \hat{U}_{v} \ + \ \hat{U}_{v}^{T}\tilde{K}^{-1} \hat{U}_{v} } \rbrace \label{eq:post_target}
\end{align*}
Therefore, to compute $\mathbf{\tilde{K}}$ and $\mathbf{\hat{U}}$, one simply has to identify the terms in $\mathbf{U}$ between the above expressions, as all the terms which don't depend on $\mathbf{U}$ are compensated  by the likelihood (denominator of Baye's formula). Proceeding as such, we obtain:
\begin{equation*}\begin{cases}
\mathbf{\tilde{K}^{-1}} &= \mathbf{ Diag(\mathbf{K}_{i}^{-1}) \ + \ H^{T}\Sigma^{-1}H \otimes I_{n}} \\
\mathbf{\tilde{K}^{-1} \hat{U}_{v}} &= \mathbf{(H^{T}\Sigma^{-1} \otimes I_{n})Y_{v}} = vec(\mathbf{Y^{T} \Sigma^{-1}H})
\end{cases}\end{equation*} 
By definition, $\hat{U}_{v}$ is the mean vector of $p(\mathbf{U_{v}}|\mathbf{Y})$ and $\mathbf{\tilde{K}}$ is its covariance matrix, hence the result.\\
\end{proof}

\paragraph{Prop \ref{prop:estimators}}
\begin{proof}
We start with the estimated variance. By hypothesis of the model, the $u_{i}$'s follow GPs of kernels $k_{i}$'s, so we have by lemma \ref{lem:cond_expect}: 
\begin{align*}
\mathbb{E}(\hat{\mathbf{u}}_{*}|\mathbf{Y}) &= \mathbf{ Diag(k_{i*}^{T}) \, Diag(K_{i}^{-1})} \, \mathbb{E}(\mathbf{U}_{v}|\mathbf{Y}) \\
&= \mathbf{ Diag(k_{i*}^{T}) \, Diag(K_{i}^{-1}) \, \left[ Diag(K_{i}^{-1}) \, + \, H^{T}\Sigma^{-1}H \otimes I_{n} \right]^{-1} (H^{T}\Sigma^{-1} \otimes I_{n})Y_{v}} \\
&= \mathbf{ Diag(k_{i*}^{T}) \, \left[ Diag(K_{i}) \, + \, (H^{T}\Sigma^{-1}H)^{-1} \otimes I_{n} \right]^{-1} \left( (H^{T}\Sigma^{-1}H)^{-1}\otimes I_{n} \right) (H^{T}\Sigma^{-1} \otimes I_{n})Y_{v}} \\
&= \mathbf{ Diag(k_{i*}^{T}) \, \left[ Diag(K_{i}) \, + \, (H^{T}\Sigma^{-1}H)^{-1} \otimes I_{n} \right]^{-1} vec \left( Y^{T}\Sigma^{-1}H(H^{T}\Sigma^{-1}H)^{-1} \right) } \,
\end{align*}

where the second equality is from proposition \ref{eq:post_U}, the third one from the matrix equality $\mathbf{ A^{-1}(A^{-1} + B^{-1})^{-1}} = \mathbf{(A+B)^{-1}B}$ and the fourth one from basic facts \ref{fact1} and \ref{fact2}.\\ 

We now compute the estimated variance from the law of total variance, conditionning on $\mathbf{U}$: $\mathbb{V}(\hat{\mathbf{u}}_{*}|\mathbf{Y}) =  \mathbb{E}_{U|Y}\left[ \mathbb{V}(\hat{\mathbf{u}}_{*}|\mathbf{Y, U}) \right] \ + \ \mathbb{V}_{U|Y} \left[ \mathbb{E}(\hat{\mathbf{u}}_{*}|\mathbf{Y, U}) \right]$, where the subscript $U|Y$ means that $\mathbb{E}_{U|Y}\left( f(\mathbf{U}) \right) = \int f(\mathbf{U}) \, p(\mathbf{U|Y})\, \mathbf{dU}$. From there, we omit this subscript and treat each term separately. We also recall that $ p( \mathbf{u^{*} | Y,U} ) $ doesn't depend on $\mathbf{Y}$ (conditionnally on their corresponding latent values, observed values of a GP are independent on all other variables of the model, and therefore $ \mathbb{E}_{u^{*}}( \mathbf{u^{*} | Y,U} ) $ doesn't depend on $\mathbf{Y}$ (idem for the variance). It comes: 
\begin{align*}
\mathbb{E}\left[ \mathbb{V}(\hat{\mathbf{u}}_{*}|\mathbf{\cancel{\mathbf{Y}}, U}) \right] &= \mathbb{E}\left[ \mathbf{Diag}(k_{i**}) - \mathbf{ Diag(k_{i*}^{T}) \, Diag(K_{i}^{-1}) \, Diag(k_{i*})} \right] \\
 &= \left( \mathbf{Diag}(k_{i**}) - \mathbf{ Diag(k_{i*}^{T}) \, Diag(K_{i}^{-1}) \, Diag(k_{i*})} \right)
\end{align*}
where the second equality holds because the term inside the expectation is deterministic. For the second term, we have:
\begin{align*}
& \mathbb{V}_{U|Y} \left[ \mathbb{E}(\hat{\mathbf{u}}_{*}|\mathbf{Y, U}) \right] = \mathbb{E}_{U|Y} \left[ \mathbb{E} \left( (\hat{\mathbf{u}}_{*}|\mathbf{\cancel{\mathbf{Y}}, U})^{2} \right) \right] \ - \ \mathbb{E}_{U|Y} \left[ \mathbb{E} \left( (\hat{\mathbf{u}}_{*}|\mathbf{\cancel{\mathbf{Y}}, U}) \right) \right]^{2} \\
&= \mathbf{ Diag(k_{i*}^{T}) \, Diag(K_{i}^{-1}) \,} \mathbb{E}(\mathbf{U}_{v}^{2}|\mathbf{Y}) \mathbf{\, Diag(K_{i}^{-1}) \, Diag(k_{i*})}\\
 & \quad - \ \mathbf{ Diag(k_{i*}^{T}) \, Diag(K_{i}^{-1}) \,} \mathbb{E}(\mathbf{U}_{v}|\mathbf{Y})^{2} \mathbf{\, Diag(K_{i}^{-1}) \, Diag(k_{i*})} \\
&= \mathbf{ Diag(k_{i*}^{T}) \, Diag(K_{i}^{-1}) \,} \mathbb{V}(\mathbf{U}_{v}|\mathbf{Y}) \mathbf{\, Diag(K_{i}^{-1}) \, Diag(k_{i*})} \\
&= \mathbf{ Diag(k_{i*}^{T}) \, Diag(K_{i}^{-1}) \, \left[ Diag(\mathbf{K}_{i}^{-1}) \, + \, H^{T}\Sigma^{-1}H \otimes I_{n} \right]^{-1} \, Diag(K_{i}^{-1}) \, Diag(k_{i*})}
\end{align*}
where the second equality stems from a relation analog to this of lemma \ref{lem:cond_expect}, and the fourth is from proposition \ref{eq:post_U}.\\

Combining the two terms, it comes: $\mathbb{V}(\hat{\mathbf{u}}_{*}|\mathbf{Y})$
{\tiny
\begin{multline*}
 = \mathbf{Diag}(k_{i**}) \\ - \mathbf{ Diag(k_{i*}^{T}) \left[ Diag(K_{i}^{-1}) - Diag(K_{i}^{-1})\left( Diag(\mathbf{K}_{i}^{-1}) + H^{T}\Sigma^{-1}H \otimes I_{n} \right)^{-1}Diag(K_{i}^{-1}) \right] Diag(k_{i*})} \\
= \mathbf{Diag}(k_{i**}) - \mathbf{ Diag(k_{i*}^{T}) \left[ Diag(K_{i}) \ + \  (H^{T}\Sigma^{-1}H)^{-1} \otimes I_{n} \right]^{-1} Diag(k_{i*}) }
\end{multline*}
}
where the second equality is from the Woodburry-type identity $\mathbf{A^{-1} - A^{-1}(A^{-1}+ B^{-1})^{-1}A^{-1}} = \mathbf{(A+B)^{-1}}$.
\end{proof}

\paragraph{Prop \ref{prop:proj_mll} to \ref{prop:interpretation}}
\textit{Propositions 3 and 5 are proven in appendixes D, E and F of \cite{OILMM}. We add a few supplementary clarifications:}
\begin{itemize}
\item Proof of $\left( \mathbf{TY} \text{ sufficient statistic for }\mathbf{U} \right) \Rightarrow \left( p(\mathbf{U|Y}) = p(\mathbf{U|TY}) \right)$: \begin{dmath}
p(\mathbf{U|Y}) = \dfrac{p(\mathbf{U,Y})}{p(\mathbf{Y})} = \dfrac{p(\mathbf{U,Y, TY})}{p(\mathbf{Y})} = \dfrac{p(\mathbf{Y|TY,\cancel{\mathbf{U}}})p(\mathbf{TY,U})}{p(\mathbf{Y})} = \dfrac{p(\mathbf{Y|TY})p(\mathbf{TY,U})}{p(\mathbf{Y})} \\ = \dfrac{p(\mathbf{Y,TY})p(\mathbf{TY,U})}{p(\mathbf{TY})p(\mathbf{Y})} = p(\mathbf{TY|Y})\dfrac{p(\mathbf{TY,U})}{p(\mathbf{TY})} = 1 \times p(\mathbf{U|TY})
\end{dmath}
where the fourth equality is precisely the definition of $\mathbf{TY}$ being a sufficient statistic of $\mathbf{U}$ for the data $\mathbf{Y}$: conditionally on $\mathbf{TY}$, the probability of the data doesn't depend on $\mathbf{U}$.
\item Proof of $\mathbf{TY|U} \sim \mathcal{N}(\mathbf{U_{v}, \, \Sigma_{P} \otimes I_{n}})$: by definition of the LMC, is stands $\mathbf{Y_{v}|U} \sim \mathcal{N}\left(\mathbf{(H \otimes I_{n}) U_{v}, \, \Sigma \otimes I_{n} }\right)$, so:
\begin{equation*}
p(\mathbf{TY|U}) = p\left(\mathbf{(T \otimes I_{n})Y_{v}|U}\right) = \mathcal{N}\left(\mathbf{(T \otimes I_{n})(H \otimes I_{n})U_{v}, \, T\Sigma T^{T} \otimes I_{n}}\right) = \mathcal{N}(\mathbf{U_{v}, \, \Sigma_{P} \otimes I_{n}})
\end{equation*}
because $\mathbf{TH = I_{q}}$.
\end{itemize}

Proof of proposition 4 is essentially trivial, all involved matrices being block-diagonal if the DPN condition stands (in particular, $\mathbf{ \mathcal{K} =  Diag(K_{i}) \, + \,  \Sigma_{P} \otimes I_{n}} = \mathbf{Diag(K_{i} + } \sigma_{i}^{2} \mathbf{I_{n}) }$). Notice that for any block-diagonal matrix $\mathbf{\tilde{D} = Diag(D_{i})}$ and any matrix $\mathbf{M}$ such that the blocksize of $\mathbf{\tilde{D}}$ matches the column size of $\mathbf{M}$, we have $\mathbf{\tilde{D}\, vec(M) = Diag(D_{i}\, M_{i})}$, where $\mathbf{M_{i}}$ is the i-th column of $\mathbf{M}$; this explains the $\mathbf{T_{i}}$'s present in the final expression.\\

\paragraph{Prop \ref{prop:facto_Q} to \ref{prop:final_mll}}
\begin{lemma}\label{lem:decom_Q}
Any symmetric matrix $\mathcal{S}$ can be decomposed as $\mathcal{S}=$ $\mathbf{QAQ^{T}} +$ $\mathbf{Q_{\bot}BQ_{\bot}^{T}} +$ $\mathbf{QCQ_{\bot}^{T}} +$ $\mathbf{Q_{\bot}C^{T}Q^{T}}$, with $\mathbf{A, B}$ symmetric matrices and $\mathbf{C}$ a $q \times (p-q)$ matrix.
\end{lemma}

\begin{proof}
It suffices to write that $\mathcal{S} = (\mathbf{QQ^{T} + Q_{\bot}Q_{\bot}^{T}})\, \mathcal{S}\, (\mathbf{QQ^{T} + Q_{\bot}Q_{\bot}^{T}}) $ (because $\mathbf{QQ^{T}}$ and $\mathbf{Q_{\bot}Q_{\bot}^{T}}$ are supplementary orthogonal projectors), and then set $\mathbf{A = Q^{T}\mathcal{S}Q, \ B = Q_{\bot}^{T}\mathcal{S}Q_{\bot}, \ C = Q^{T}\mathcal{S}Q_{\bot}}$. Another way of seing this is by noticing that $\mathbf{A, \ B, \ C \ \text{and} \ C^{T}}$ are the blocks of the representation of $\mathcal{S}$ in the basis spanned by $\mathbf{Q}$ and $\mathbf{Q_{\bot}}$.
\end{proof}

\begin{rem}
Notice that $\mathbf{A}$, $\mathbf{B}$ and $\mathbf{C}$ are simply subblocks of $\mathbf{\Sigma^{-1}}$ in the basis spanned by $\mathbf{Q}$ and $\mathbf{Q_{\bot}}$, and that $\mathbf{B}$ and $\mathbf{C}$ remain arbitrary: the DPN condition only involves a $q \times q$ submatrix of $\mathbf{\Sigma^{-1}}$, with $q \ll p$ in most cases. This observation, which hints at the benign nature of this assumption, is made more precise in \ref{an:restrictivity}, where a closed-form expression of the optimal $\mathbf{D}$ and a noise correction procedure are provided.
\end{rem}

\begin{prop}\label{prop:charac_Q}
Let $\mathbf{\Sigma^{-1}}=$ $\mathbf{QAQ^{T}} +$ $\mathbf{Q_{\bot}BQ_{\bot}^{T}} +$ $\mathbf{QCQ_{\bot}^{T}} +$ $\mathbf{Q_{\bot}C^{T}Q^{T}}$ be a decomposition of $\mathbf{\Sigma^{-1}}$ as in lemma \ref{lem:decom_Q}. Then $\mathbf{\Sigma_{P} = R^{-1}A^{-1}R^{-T}}$, (where $\mathbf{R^{-T}}$ denotes the inverse transpose of $\mathbf{R}$) and the DPN condition is equivalent to ($\mathbf{A = R^{-T}D R^{-1}}$ for some diagonal matrix $\mathbf{D}$ of size $q$).
\end{prop}

\begin{proof}
We have $\mathbf{H^{T}\Sigma^{-1}H = R^{T}Q^{T}\Sigma^{-1}QR}$; moreover, $\mathbf{Q^{T}Q_{\bot} = 0}$ , $\mathbf{Q_{\bot}^{T}Q = 0}$ and $\mathbf{Q^{T}Q = I_{q}}$ because the columns of $\mathbf{Q}$ and $\mathbf{Q_{\bot}}$ are mutually orthonormal. Thus we see that the only term from the decomposition of $\Sigma^{-1}$ that is preserved by left- and right-multiplication by $\mathbf{Q}$ and $\mathbf{Q}^{T}$ is $\mathbf{A}$: $\mathbf{H^{T}\Sigma^{-1}H = R^{T}A R} \Leftrightarrow \mathbf{\Sigma_{P} = R^{-1}A^{-1} R^{-T} }$. Moreover, ($\mathbf{H^{T}\Sigma^{-1}H}$ is diagonal) $\Leftrightarrow$ ($\mathbf{R^{T} A R = D}$ for some diagonal matrix $\mathbf{D}$) $\Leftrightarrow$ ($\mathbf{A = R^{-T}D R^{-1}}$ for some diagonal matrix $\mathbf{D}$).
\end{proof}

\paragraph{Prop \ref{prop:facto_Q}}
\begin{proof}
The proof is by direct calculation. Starting from the suggested factorized form, one easily arrives at:\\ $\mathbf{Q_{+}R_{+}^{-T}D_{+}^{-1}R_{+}^{-1}Q_{+}^{T}} \, = \, \mathbf{ Q R^{-T}\Sigma_{P}^{-1} R^{-1}Q^{T} } \ + \ \mathbf{ Q_{\bot} \Sigma_{\bot}^{-1} Q_{\bot}^{T}} \ + \ \mathbf{2 \, Sym(Q R^{-T} M Q_{\bot}^{T}) }$.
 From this, we can identify terms with the decomposition of lemma \ref{lem:decom_Q}(where every identification is valid) to obtain the announced expressions for $\mathbf{M}$, $\mathbf{\Sigma_{P}^{-1}}$ and $\mathbf{\Sigma_{\bot}^{-1}}$.
\end{proof}

\paragraph{Prop \ref{prop:noise_param}}\label{par:block_id}
\begin{proof}
First notice that $\mathbf{\tilde{\Sigma}_{\bot}}$ is the Schur complement of  $ \mathbf{\Sigma_{P}}^{-1}$ inside $\mathbf{D_{+}}$. It happens that providing a pair of Schur complements of a symmetric matrix and any diagonal block of itself or its inverse fully specifies this matrix. Indeed, using the previous tilded notations for the blocks of $\mathbf{\Sigma}$, blockwise matrix inversion formulas show that $\mathbf{\tilde{\Sigma_{P}} = \Sigma_{P} + \Sigma_{P} M \tilde{\Sigma_{\bot}} M^{T} \Sigma_{P} }$, $\mathbf{\tilde{M} = - \Sigma_{P} M \tilde{\Sigma_{\bot}}}$, and finally $\mathbf{\Sigma_{\bot}^{-1} = \tilde{\Sigma}_{\bot} + \tilde{\Sigma}_{\bot} \tilde{M}^{T} \Sigma_{P}^{-1} \tilde{M} \tilde{\Sigma}_{\bot} }$, so that by the last equality $\mathbf{\Sigma_{\bot}^{-1} = \tilde{\Sigma}_{\bot} + M^{T} \Sigma_{P} M}$. Thus we see that all blocks of $\mathbf{\Sigma^{-1}}$ and $\mathbf{\Sigma}$ are expressed from $\mathbf{\tilde{\Sigma_{\bot}}, \ \Sigma_{P}}$ and $\mathbf{M}$ without any circular definitions. Moreover, by a standard result, $\mathbf{\Sigma}$ is p.s.d iff its two Schur complements $\mathbf{\tilde{\Sigma}_{\bot}}$ and $\mathbf{\Sigma_{P}}$ are; see for instance \cite{convex}. 
\end{proof}

\paragraph{Prop \ref{prop:express_T}}
\begin{proof}
Applying $\mathbf{H^{T}}$ to the decomposition of $\mathbf{\Sigma^{-1}}$ from lemma \ref{lem:decom_Q} strikes out the terms beginning with $\mathbf{Q_{\bot}}$, so:
\begin{align*}
\mathbf{T} &= \mathbf{\Sigma_{P} H^{T}\Sigma^{-1}} = \ \mathbf{ \Sigma_{P} \, R^{T}Q^{T} \, ( QAQ^{T} + \cancel{\mathbf{Q_{\bot}BQ_{\bot}^{T}}} + QCQ_{\bot}^{T} + \cancel{\mathbf{Q_{\bot}C^{T}Q^{T}}} )} \\
&= \mathbf{\Sigma_{P} R^{T}A Q^{T} \, + \, \Sigma_{P} R^{T}C Q_{\bot}^{T}} \\
 &= \mathbf{\Sigma_{P} R^{T}A Q^{T} \, + \, \Sigma_{P} M Q_{\bot}^{T}} \quad \text{ by definition of } \mathbf{M}.
\end{align*}
By proposition \ref{prop:charac_Q}, $\mathbf{\Sigma_{P} = R^{-1}A^{-1}R^{-T}}$, so the first term is equal to $\cancel{\mathbf{\Sigma_{P} R^{T}}}(\cancel{\mathbf{R^{-T}\Sigma_{P}^{-1}}}  \mathbf{R^{-1}) Q^{T}} = \mathbf{R^{-1} Q^{T}} = \mathbf{H^{+} }$ (the latter equality being a standard result about the Moore-Penrose pseudoinverse). We can reformulate the second one to show that it is invariant to the choice of $\mathbf{Q_{\bot}}$: starting from its previous expression, we have:
\begin{equation*}
    \mathbf{\Sigma_{P} H^{T}\, QCQ_{\bot}^{T} = \Sigma_{P} H^{T} \, Q(Q^{T}\Sigma^{-1}Q_{\bot}) Q_{\bot}^{T} } = \mathbf{ \Sigma_{P} H^{T} P \Sigma^{-1} P_{\bot}^{T}}
\end{equation*}
with $\mathbf{P}$ the orthogonal projector onto $Span(\mathbf{Q})$ and $\mathbf{P}$ the orthogonal projector onto $Span(\mathbf{Q_{\bot}})=Span(\mathbf{Q})^{\bot}$. The claim is thus clearly proven: $\mathbf{Q_{\bot}Q_{\bot}^{T}}$ is only one possible expression of the orthogonal projector on $Span(\mathbf{Q})^{\bot}$, any other orthogonal supplement of $\mathbf{Q}$ would yield the same result.\\
\end{proof}

\paragraph{proposition \ref{prop:final_mll}}
\begin{proof}
We start from an intermediary expression taken from the appendix G of \cite{OILMM} under the DPN hypothesis:

\begin{multline}\label{eq:inter_mll}
- 2 \log p(\mathbf{Y}) = (p-q)n\log 2\pi + n\log \frac{|\mathbf{\Sigma}|}{|\mathbf{\Sigma_{P}}|} + \sum_{j=1}^{n} \mathbf{Y_{j} (\Sigma^{-1} - T^{T}\Sigma_{P}^{-1}T) Y_{j}^{T}} \\
+ \sum_{i=1}^{q} \log \mathcal{N}(\mathbf{YT_{i} | 0, \, K_{i} + } \sigma_{i} \mathbf{I_{n}}) 
\end{multline}

where the second term represents the noise lost by projection, the third is the data lost by projection and the last is the standard GP MLL of the independent latent processes. Plugging in the factorization of proposition \ref{prop:facto_Q}, we just reformulate the second and third terms on the right-hand size. Let's start with the second, corresponding to the discarded noise: using the quantities from the previous section, it suffices to note that with the suggested factorization $|\mathbf{\Sigma}| = |\mathbf{R}|^{2}|\mathbf{D_{+}}|$. Then, by Schur's determinant formula (determinant of a four-blocks matrix), $|\mathbf{D_{+}}| = |\mathbf{\tilde{\Sigma}_{\bot}}| \, |\mathbf{ \tilde{\Sigma_{P}} - \tilde{M}\tilde{\Sigma}_{\bot}\tilde{M}^{T} }| = |\mathbf{\tilde{\Sigma}_{\bot}}| \, |\mathbf{\Sigma_{P}}|$. We now address the third term. The proof of \eqref{eq:express_T} shows that $\mathbf{H^{T}\Sigma^{-1}} = \mathbf{\Sigma_{P}^{-1}T} =  \mathbf{\Sigma_{P}^{-1}R^{-1}Q^{T} + MQ_{\bot}^{T} }$, so:
\begin{align*}
\mathbf{ \Sigma^{-1} - T^{T} \Sigma_{P}^{-1} T } &= \mathbf{ \Sigma^{-1} - \Sigma^{-1} H \Sigma_{P} H^{T} \Sigma^{-1} } \\ 
&= \mathbf{ \Sigma^{-1} - ( QR^{-T}\Sigma_{P}^{-1} + Q_{\bot}M^{T})\, \Sigma_{P} \, (\Sigma_{P}^{-1}R^{-1}Q^{T} + MQ_{\bot}^{T}) }\\
&= \mathbf{ \Sigma^{-1} - Q_{+}} \left( \begin{array}{c|c} \mathbf{ R^{-T}\Sigma_{P}^{-1}\Sigma_{P}\Sigma_{P}^{-1}R^{-1}} & \mathbf{R^{-T}}\cancel{\mathbf{\Sigma_{P}^{-1}}}\cancel{\mathbf{\Sigma_{P}}} \mathbf{M} \\ \hline \mathbf{M^{T}} \cancel{\mathbf{\Sigma_{P}}}\cancel{\mathbf{\Sigma_{P}^{-1}}} \mathbf{R^{-1}}  & \mathbf{M^{T}\Sigma_{P}M} \end{array} \right) \mathbf{ Q_{+}^{T} } \\
&= \mathbf{ Q_{+} R_{+}^{-T}} \left( \begin{array}{c|c} \cancel{\mathbf{\Sigma_{P}^{-1}}} & \cancel{\mathbf{M}} \\ \hline \cancel{\mathbf{M^{T}}} & \mathbf{B} \end{array} \right) \mathbf{ R_{+}^{-1} Q_{+}^{T}} - \mathbf{ Q_{+} R_{+}^{-T}} \left( \begin{array}{c|c} \cancel{\mathbf{\Sigma_{P}^{-1}}} & \cancel{\mathbf{M}} \\ \hline \cancel{\mathbf{M^{T}}} & \mathbf{M^{T} \Sigma_{P} M} \end{array} \right) \mathbf{ R_{+}^{-1} Q_{+}^{T}} \\
&= \mathbf{Q_{\bot} (\Sigma_{\bot}^{-1} - M^{T} \Sigma_{P} M) Q_{\bot}^{T} } \\
&= \mathbf{Q_{\bot} \tilde{\Sigma}_{\bot}^{-1} Q_{\bot}^{T} } \quad \text{by blockwise inversion formula.}
\end{align*}
To see that the obtained expression doesn't depend on a particular choice of $\mathbf{Q_{\bot}}$, we just have to prove it for $\mathbf{Q_{\bot} \tilde{\Sigma}_{\bot}^{-1} Q_{\bot}^{T} }$, as we have already shown it for $\mathbf{T}$ in proposition \ref{prop:express_T}. Let's write that $\mathbf{\tilde{\Sigma_{\bot}}} = \mathbf{Q_{\bot}^{T} \Sigma Q_{\bot}}$ (see corollary \ref{cor:facto_Q}). Recall that any orthonomal complement of $\mathbf{Q}$ denoted as $\mathbf{Q^{'}}$ can be written $\mathbf{Q^{'} = Q_{\bot}W}$ for some square orthonormal matrix $\mathbf{W}$. Then: $\mathbf{Q^{'} ( Q^{'T} \Sigma Q^{'})^{-1} Q^{'T}} \ = \ \mathbf{Q_{\bot}W ( W^{T}Q_{\bot}^{T} \Sigma Q_{\bot}W)^{-1} W_{T}Q_{\bot}^{T}} = \mathbf{Q_{\bot}W W^{T} (Q_{\bot}^{T} \Sigma Q_{\bot})^{-1} WW^{T}Q_{\bot}^{T}} = \mathbf{ Q_{\bot} ( Q_{\bot}^{T} \Sigma Q_{\bot})^{-1} Q_{\bot}^{T}}$. 
\end{proof}

\paragraph{Prop \ref{prop:diagonal_B}}
\begin{proof}
proposition \ref{prop:estimators} shows that the estimated mean and variance depend on $\mathbf{\Sigma}$ through $\mathbf{T}$ and $\mathbf{\Sigma_{P}}$ only. proposition \ref{prop:noise_param} states that $\mathbf{(\tilde{\Sigma_{\bot}}, \, \Sigma_{P}, \, M)}$ is a complete and independent parametrization of $\mathbf{\Sigma}$, so $\mathbf{T = R^{-1}Q^{T} + \Sigma_{P}MQ_{\bot}^{T}}$ (see proposition \ref{eq:express_T}) doesn't depend on $\mathbf{\tilde{\Sigma}_{\bot}}$ and neither do the posteriors.\\

For the second affirmation, notice that the likelihood is invariant under the transformation $\mathbf{\tilde{\Sigma}_{\bot} \leftarrow W \tilde{\Sigma}_{\bot} W^{T}}$, $\mathbf{Q_{\bot} \leftarrow Q_{\bot} W^{T}}$ with $\mathbf{W}$ orthonormal, as in particular it doesn't affect $\mathbf{T}$ (which by proposition \ref{eq:express_T} only depends on $\text{Im}(\mathbf{Q_{\bot}})$ and not a specific vector basis); moreover, such a transformation of $\mathbf{Q_{\bot}}$ doesn't affect posteriors of the model either, as it only appears in $\mathbf{T}$ in their expression. Therefore, we can always select the matrix $\mathbf{W}$ which diagonalizes $\mathbf{\tilde{\Sigma}_{\bot}}$, i.e we can jointly optimize $\mathbf{\tilde{\Sigma}_{\bot}}$ and $\mathbf{Q_{\bot}}$ while enforcing the diagonality of $\mathbf{\tilde{\Sigma}_{\bot}}$.\\
\end{proof}

\subsection{Matrix proof of proposition \ref{prop:estimators}}\label{sec:matrix_der}
\textit{We here give another proof of proposition \ref{prop:estimators}, which starts from the naive expression of LMC posteriors and uses only matrix algebra. It is a good example of involved manipulations of Kronecker products.}

\begin{proof}
We start with the "naive" expression of the LMC estimators, as stated in section 1: 

\begin{align}
\mathbb{E}(\hat{\mathbf{y}}_{*}|\mathbf{Y}) &= \mathbf{k_{*}^{T} \, \mathcal{K}^{-1} \, Y_{v}} \\
\mathbb{V}(\hat{\mathbf{y}}_{*}|\mathbf{Y}) &= \mathbf{k_{**} - k_{*}^{T} \, \mathcal{K}^{-1} k_{*}}
\end{align}
where the kernel considered here is the matrix kernel of the vector-valued GP: $\mathbf{k(x, x^{'})} = \sum_{i=1}^{q}\mathbf{B}_{i}\, k_{i}(\mathbf{x, x^{'}})$, and we recall that $\mathbf{K = (H \otimes I_{n}) \, Diag(K_{i}) \, (H^{T} \otimes I_{n})}$\footnote{This can be seen by writing explicitly the coordinates of $\mathbf{K} = \sum_{i=1}^{q}\mathbf{B_{i} \otimes K_{i}}$, reminding that $\mathbf{B_{i} = H_{i}H_{i}^{T}}$ (where $\mathbf{H_{i}}$ is the $i$-th column of $\mathbf{H}$).}. Similarly, we have $\mathbf{k_{*} = (H \otimes I_{n}) \, Diag(k_{i}) \, H^{T}}$, and $\mathbf{k_{**} = H \, Diag(k_{i**}) \, H^{T}}$. Anticipating on the cumbersome writing of Kronecker products, we adopt the following convention: for any matrix $\mathbf{M, \ \tilde{M} \equiv M \otimes I_{n}}$. The properties of the Kronecker product (fact \eqref{fact1}) ensures us that $\forall \mathbf{M, \ N, \ \tilde{M} \tilde{N}} = \largetilde{\mathbf{MN}}$. We start with the expected mean:

\begin{align*}
\mathbf{k_{*}^{T} \, \mathcal{K}^{-1} \, Y_{v}} &= \mathbf{H \, Diag(k_{i}^{T}) \, \tilde{H}^{T} \left[ \tilde{H}  \, Diag(K_{i}) \, \tilde{H}^{T} + \tilde{\Sigma} \right]^{-1} Y_{v}} \\
&= \mathbf{ H \, Diag(k_{i}^{T}) \, \tilde{H}^{T} \left[ \tilde{\Sigma}^{-1} - \tilde{\Sigma}^{-1} \tilde{H} \left(  Diag(K_{i}^{-1}) + \tilde{H^{T}} \tilde{\Sigma}^{-1} \tilde{H} \right)^{-1} \tilde{H}^{T}\tilde{\Sigma}^{-1} \right] Y_{v}} \\
&= \mathbf{ H \, Diag(k_{i}^{T}) \, \left[ \tilde{H}^{T} \tilde{\Sigma}^{-1} - \tilde{H}^{T} \tilde{\Sigma}^{-1} \tilde{H} \left( Diag(K_{i}^{-1}) + \tilde{\Sigma}_{P}^{-1} \right)^{-1} \tilde{H}^{T} \tilde{\Sigma}^{-1} \right] Y_{v}} \\
&= \mathbf{ H \, Diag(k_{i}^{T}) \, \left[ \tilde{I_{q}} - \tilde{\Sigma}_{P}^{-1} \left(  Diag(K_{i}^{-1}) + \tilde{\Sigma}_{P}^{-1} \right)^{-1} \right] \,\tilde{H}^{T} \tilde{\Sigma}^{-1} Y_{v}} \\
&= \mathbf{ H \, Diag(k_{i}^{T}) \, \left[ \tilde{\Sigma}_{P}^{-1} - \tilde{\Sigma}_{P}^{-1} \left[  Diag(K_{i}^{-1}) + \tilde{\Sigma}_{P}^{-1} \right]^{-1} \tilde{\Sigma}_{P}^{-1} \right] \, \tilde{\Sigma}_{P} \, \tilde{H}^{T} \tilde{\Sigma}^{-1} \, Y_{v}} \\
&= \mathbf{ H \, Diag(k_{i}^{T}) \, \left[ Diag(K_{i}) + \tilde{\Sigma}_{P} \right]^{-1} \, \tilde{\Sigma}_{P} \tilde{H}^{T} \tilde{\Sigma}^{-1} \, vec(Y)} \\
&= \mathbf{ H \, Diag(k_{i}^{T}) \, \left[ Diag(K_{i}) + \Sigma_{P} \otimes I_{n} \right]^{-1} \, vec(Y  \Sigma^{-1} H \Sigma_{P}^{-1})}
\end{align*}

The first equality is from the above considerations; the second is from Woodburry's identity; the sixth is a backward application of Woodburry's identity; and the last is an application of fact \eqref{fact2}.\\

We now turn to the expected variance, setting further notations $\mathbb{K} = \mathbf{Diag(K_{i})}$ and $\mathfrak{K} = \mathbf{Diag(k_{i})}$ for compacity:

\begin{align*}
\mathbf{k_{**} - k_{*}^{T} \, \mathcal{K}^{-1} k_{*}} &= \mathbf{H \, Diag(k_{i**}) \, H^{T} - H \, \mathfrak{K}^{T} \, \tilde{H}^{T} \left( \tilde{H} \, \mathbb{K} \, \tilde{H}^{T} + \tilde{\Sigma} \right)^{-1} \, \tilde{H} \, \mathfrak{K} \, H^{T} } \\
&= \mathbf{H \, \left[ Diag(k_{i**}) - \mathfrak{K}^{T} \, \tilde{H}^{T} \left( \tilde{H} \, \mathbb{K} \, \tilde{H}^{T} + \tilde{\Sigma} \right)^{-1} \, \tilde{H} \, \mathfrak{K} \right] H^{T} } \\
&= \mathbf{H \, \left[ Diag(k_{i**}) - \mathfrak{K}^{T} \left( \mathbb{K} + \tilde{\Sigma}_{P} \right)^{-1} \mathfrak{K} \right]  H^{T} } \\
&= \mathbf{H} \left[ \mathbf{Diag}(k_{i**}) - \mathbf{ Diag(k_{i*}^{T}) \left[ Diag(K_{i}) \ + \  (H^{T}\Sigma^{-1}H)^{-1} \otimes I_{n} \right]^{-1} Diag(k_{i*})} \right] \mathbf{H^{T}}
\end{align*}

The first equality is from the above considerations, and the third is lemma \ref{fact4}.\\
\end{proof}

\section{Restrictivity of the DPN hypothesis}\label{an:restrictivity}

\begin{prop}\label{prop:noise_restrictivity}
Let $\mathbf{\Sigma_{opt}}$ the value of $\mathbf{\Sigma}$ which optimizes the MLL of the unconstrained LMC model. Let $ \mathbf{\Sigma_{opt}^{-1} = QAQ^{T} + Q_{\bot}BQ_{\bot}^{T} + 2 Sym(QCQ_{\bot}^{T}) }$ be the decomposition of $\mathbf{\Sigma_{opt}^{-1}}$ as in proposition \ref{lem:decom_Q}. Then the minimal distance between $\mathbf{\Sigma_{opt}^{-1}}$ and a precision matrix $\mathbf{\Sigma_{app}^{-1}}$ compatible with the DPN hypothesis is:
\begin{equation}\label{eq:noise_discrepancy}
 \min_{\mathbf{\Sigma_{app}}} || \mathbf{\Sigma_{opt}^{-1} - \Sigma_{app}^{-1}}||_{F}^{2} \ = \ \min_{\mathbf{D}} \mathbf{|| A - R^{-T}D R^{-1}||_{F}^{2}} \quad \text{subject to} \ D \ \text{positive and diagonal}, 
\end{equation}
where $||\cdot||_{F}$ denotes the Frobenius norm. It is minimized for:

\begin{equation}
\mathbf{\Sigma_{app}^{-1}} = \mathbf{QR^{-T}D^{'} R^{-1}Q^{T}} + \mathbf{Q_{\bot}BQ_{\bot}^{T}} + \mathbf{2 Sym(QCQ_{\bot}^{T})}
\end{equation}
where $\mathbf{D^{'} = Diag \left[ (R^{-1}R^{-T} \odot R^{-1}R^{-T})^{-1} diag(R^{-1} A R^{-T}) \right] }$ is the optimal $\mathbf{D}$ in the above expression, $\mathbf{diag}$ is the operator taking the diagonal of a square matrix (into vector form) and $\odot$ denotes the Hadamard product (elementwise matrix product).\\
\end{prop}

\begin{proof}
Let's write $\mathbf{\Sigma_{app}^{-1} = QA^{'}Q^{T} + Q_{\bot}B^{'}Q_{\bot}^{T} + 2 Sym(QC^{'}Q_{\bot}^{T})}$ the decomposition of $\mathbf{\Sigma_{app}^{-1}}$. It is a general property that the Frobenius norm of a symmetric matrix can be computed blockwise, that is: 
\begin{align*}
\min_{\mathbf{\Sigma_{app}}} || \mathbf{\Sigma_{opt}^{-1} - \Sigma_{app}^{-1} }||_{F}^{2}  \ &= \mathbf{ \ Tr \left( ( \Sigma_{opt}^{-1} - \Sigma_{app}^{-1})^{2} \right) } \\
 &= \ \mathbf{ Tr \left( (A - A^{'})^{2} \right) + Tr \left( (B - B^{'})^{2} \right) + 2 \, Tr \left( (C - C^{'})(C - C^{'})^{T} \right) }
\end{align*}
This can be shown either by writing the norm as a sum of squared coefficients and splitting the sum at the right indices, or by replacing $\mathbf{\Sigma_{opt}^{-1}}$ and $\mathbf{\Sigma_{app}^{-1}}$ by their decompositions, expanding the trace of their squared difference and eliminating all cross-coefficients because of the circularity of the trace (i.e: $\mathbf{Tr(Q X Q_{\bot}^{T}) = Tr(X Q_{\bot}^{T} Q) = 0}$).\\
proposition \ref{prop:charac_Q} showed that the DPN hypothesis doesn't put any constraints on $\mathbf{B}$ and $\mathbf{C}$. Therefore, one can simply set $\mathbf{B^{'} = B}$ and $\mathbf{C^{'} = C}$ to zero their contribution to the discrepancy. Proposition \ref{prop:charac_Q} also shows that the DPN hypothesis is equivalent to ($\mathbf{A^{'} = R^{-T}D R^{-1}}$ for some diagonal matrix $\mathbf{D}$), hence equation \eqref{eq:noise_discrepancy}.\\

In order to find the optimal $\mathbf{D}$, we set additional notations: $\mathbf{D = Diag(v)}$; $\mathbf{|| A - R^{-T}D R^{-1}||_{F}^{2} = || \Delta ||_{F}^{2} = \Phi (\Delta) }$; $\mathbf{M:N = Tr(M^{T} N)}$ denotes the Frobenius scalar product. We then compute the differential, considering $\mathbf{v}$ as the only variable:
\begin{align*}
\mathbf{d\Phi} \ &= \ \mathbf{ 2 \, \Delta: d\Delta } \ = \ \mathbf{ 2 \, \Delta: d (R^{-T} Diag(v)\,  R^{-1}) } = \mathbf{ 2 \, \Delta: R^{-T} d(Diag(v)) \, R^{-1} } \\
& = \ \mathbf{ 2 \, R^{-1} \Delta R^{-T}: d ( Diag(v) ) } \ = \ \mathbf{ 2 \, R^{-1} \Delta R^{-T}: Diag(dv) } \ = \ \mathbf{ 2 \, diag(R^{-1} \Delta R^{-T}): dv }
\end{align*} 
All of these expressions are standard manipulations involving differentials and scalar products; see for instance \cite{matrix_cookbook}. The last one, $\mathbf{X: Diag(y) = diag(X): y}$ \footnote{Where we recall that $\mathbf{diag(X)}$ denotes the diagonal of matrix $\mathbf{X}$ while $\mathbf{Diag(y)}$ is the diagonal square matrix made from the vector $\mathbf{y}$.} is less standard: it can be proved by direct expansion in matrix coefficients.\\
The last expression defines the gradient of $\mathbf{\Phi}$ with respect to the vector $\mathbf{v}$: $\mathbf{d\Phi = 2 \, diag(R^{-1} \Delta R^{-T}): dv } \Leftrightarrow \mathbf{ \nabla_{v}\Phi = diag(R^{-1} \Delta R^{-T}) }$. We can then set this gradient to zero to find the global minimum:

\begin{align*}
\mathbf{ \nabla_{v}\Phi = 0 } \ & \Leftrightarrow \ \mathbf{ diag(R^{-1} \Delta R^{-T}) = 0 } \ \Leftrightarrow \ \mathbf{ diag(R^{-1} A R^{-T}) = diag(R^{-1} R^{-T} Diag(v) R^{-1} R^{-T}) } \\
& \ \Leftrightarrow \ \mathbf{ diag(R^{-1} A R^{-T}) = \left( (R^{-1} R^{-T}) \odot (R^{-1} R^{-T}) \right)v } \\
& \ \Leftrightarrow \ \mathbf{v = (R^{-1}R^{-T} \odot R^{-1}R^{-T})^{-1} diag(R^{-1} A R^{-T}) } 
\end{align*}

The second-to-last equivalence is from another unusual matrix identity: $\mathbf{diag(X \, Diag(v) Y) = (Y^{T} \odot X)v }$. Here again, the proof is by direct expansion of the matrix coefficients.
\end{proof}

Proposition \ref{prop:noise_restrictivity} proves what was previously stated in a more informal way: the discrepancy between a DPN-restricted LMC and an unrestricted one lies only in an approximation of the "true" matrix $\mathbf{A}$. The fact that the minimizer of this discrepancy has a simple form paves the way for an eventual \emph{correction procedure} for the noise: if the data is assumed to exhibit a specific noise structure which is not diagonally projectable, one could parametrize such an optimal noise $\mathbf{\Sigma_{opt}}$ in addition to the noise parameters $\mathbf{\tilde{B}, \ M}$ and $\mathbf{\Sigma_{P}}$. One would then add to the likelihood a term proportional to $n \times \mathbf{diag(Q^{T} \Sigma_{opt}^{-1} Q - R^{-T} \Sigma_{P}^{-1} R^{-1}) }$ (reminding that $\mathbf{D = \Sigma_{P}^{-1} }$ and $\mathbf{A = Q^{T} \Sigma^{-1} Q}$), which would warp the model towards a more realistic noise.\\
Alternately, an interleaved optimization scheme could be developed, in the spirit of a projected gradient descent: at each iteration, one would estimate a DPN-compatible noise $\mathbf{\Sigma_{app}}$ by optimizing the MLL of proposition \ref{prop:final_mll}, find the nearest noise $\mathbf{\Sigma_{opt}}$ (or at least a near one) enforcing the desired noise structure, compute its decomposition in terms of $\mathbf{A, \ B}$ and $\mathbf{C}$, and then come back to the initial variables by setting $\mathbf{B \leftarrow B}$, $\mathbf{C \leftarrow C}$ and $\mathbf{\Sigma_{P}^{-1}} = \mathbf{Diag(v) \leftarrow Diag \left[ (R^{-1}R^{-T} \odot R^{-1}R^{-T})^{-1} diag(R^{-1} A R^{-T}) \right] }$. This is beyond the scope of this article and has not been experimented on yet, by lack of a relevant study case.\\

\section{Comparison with other LMC approximations}\label{an:comparison}

	\subsection{Relation to variational approaches}\label{sec:variational}

We wish to underline the analogy between the posterior estimates of proposition \ref{prop:estimators} and these of variational LMC models, which can be found for instance in \cite{hetero_MOGP} or \cite{NSVLMC}. If $\mathbf{f_{i}}$ represents the values of latent process $i$ at observation points, and $\mathbf{u_{i}}$ its values at pseudo-input points, variational models introduce approximate posteriors $q(\mathbf{u_{i}}) = \mathcal{N}(\mathbf{u_{i}|m_{i}, S_{i}})$ ($\mathbf{m_{i}, S_{i}}$ being parameters of the model) such that $q(\mathbf{u_{i}}) \simeq p(\mathbf{u_{i}, f_{i}|y})$. The key point here is that \textbf{it is always assumed (but not always emphasized) that the approximate posterior $q(\mathbf{u})$ of the full model factorizes over latent processes: $q(\mathbf{u}) = \prod_{i=1}^{q}q(\mathbf{u_{i}})$}. This feature, crucial to the PLMC, is here also assumed as a part of a larger approximation. This construction makes all posteriors and likelihood terms factorize too: $p(\mathbf{u_{i}, f_{i}|y}) \simeq q(\mathbf{f_{i}}) = \mathcal{N}(\mu_{i}, \nu_{i})$, with

\begin{align*}
\mu_{i} &= \mathbf{k_{i*}^{T} \, K_{i}^{-1} \, m_{i}} \\
\nu_{i} &= k_{i**} \, - \, \mathbf{k_{i*}^{T} \, (K_{i}^{-1} - K_{i}^{-1}S_{i}K_{i}^{-1})k_{i*}} \\
&= k_{i**} \, - \, \mathbf{k_{i*}^{T} \, \left(K_{i}^{-1} - K_{i}^{-1}\left( K_{i}^{-1} + (S_{i}^{-1} - K_{i}^{-1}) \right)^{-1} K_{i}^{-1} \right)k_{i*}} \\
&= k_{i**} \, - \, \mathbf{k_{i*}^{T} \, \left( K_{i} + (S_{i}^{-1} - K_{i}^{-1})^{-1} \right)^{-1}\, k_{i*}} \ \text{}
\end{align*}
where the covariance vector and matrix are here evaluated at pseudo-input points rather than observed datapoints. One can notice the direct similarity with the decoupled expressions of proposition 4, simply replacing $\mathbf{YT_{i}}$ by $\mathbf{m_{i}}$ and $\sigma_{i}^{2}\mathbf{I_{n}}$ by $\mathbf{(S_{i}^{-1} - K_{i}^{-1})^{-1}}$ (and the real observation points by their pseudo-input counterparts): the variational model learns additional parameters (these of the approximate posterior, $\mathbf{m_{i}}$ and $\mathbf{S_{i}}$) which substitute themselves to the projected data and projected noise.

\subsection{Comparison with the ICM}

One natural question at this point is: are the computational gains of the ICM, relative to a naive implementation of a general LMC, due to some sort of latent process decoupling ? This is plausible, as the ICM has a simpler latent structure than the LMC. Yet the answer is negative: its latent processes are \emph{not} independent conditionally on observations. Indeed, we have already shown this property to be equivalent to the DPN condition, which is not automatically enforced by the ICM, its matrices $\mathbf{H}$ and $\mathbf{\Sigma}$ remaining arbitrary. The source of computational efficiency of this model thus lies elsewhere; it is instructive to compare it to this of the PLMC.\\

The covariance matrix of the ICM is defined as $\mathbf{K = K_{T} \otimes K_{x}}$. Most computations with the model involve the noise-added covariance $\mathbf{\mathcal{K}} = \mathbf{K_{T} \otimes K_{x} + \Sigma \otimes I_{n}}$, which at first sight is not easily invertible; additional tricks or approximations are thus necessary to carry efficient computation. We comment here on the approach of \cite{all_in_the_noise}, as it is exact and more efficient than the original model of \cite{MTGP}, which only uses low-rank approximations of $\mathbf{K_{T}}$ and $\mathbf{K_{x}}$.\\

If we frame the ICM as a particular case of LMC, we get $\mathbf{K_{T} = HH^{T}}$. Indeed, the covariance matrix of the general LMC is $\mathbf{K} = \sum_{i=1}^{q}\mathbf{H_{i}H_{i}^{T}} \otimes \mathbf{K}_{i}$ (with $\mathbf{H_{i}}$ the $i$-th row of $\mathbf{H}$), so if $\mathbf{K_{i} = K_{x}} \ \forall i$ (all latent processes share the same kernel in the ICM), we get $\mathbf{K_{T}} = \sum_{i=1}^{q}\mathbf{H_{i}H_{i}^{T}} = \mathbf{HH^{T}}$. We define $\mathbf{U_{H}S_{H}^{2}U_{H}^{T}}$ as the diagonalization of $\mathbf{K_{T}}$, and $\mathbf{\Sigma = U_{\Sigma}S_{\Sigma}U_{\Sigma}^{T}}$ the diagonalization of $\mathbf{\Sigma}$\footnote{In these definitions $\mathbf{S_{H}}$ is squared for coherency with the rest of the present paper, while $\mathbf{S_{\Sigma}}$ is not to keep the notations of \cite{all_in_the_noise}.}. Then the approach of \cite{all_in_the_noise} fully lies in the following decomposition (their equation (7)\footnote{Where we have made the restriction $\mathbf{\Omega = I_{n}}$ for reading purposes, corresponding to homoskedactic noise.}):

\begin{equation}\label{eq:ICM_facto}
\mathbf{\mathcal{K} = K_{T} \otimes K_{x} + \Sigma \otimes I_{n} = \left( U_{\Sigma}S_{\Sigma}^{\frac{1}{2}} \otimes I_{n}\right) \left( S_{\Sigma}^{-\frac{1}{2}}U_{\Sigma}^{T} K_{T} U_{\Sigma}S_{\Sigma}^{-\frac{1}{2}} \otimes K_{x} \, + \, I_{p} \otimes I_{n} \right) \left( S_{\Sigma}^{\frac{1}{2}} U_{\Sigma}^{T} \otimes I_{n} \right) }
\end{equation}

The purpose of this factorization is to exploit the convenient properties of $\mathbf{I_{p} \otimes I_{n}}$: for any symmetric matrices $\mathbf{S, \, S^{'}}$ of eigen-decompositions $\mathbf{UDU^{T}}$ and $\mathbf{U^{'}D^{'}U^{'T}}$, observe that the eigen-decomposition of  $\mathbf{( S \otimes S^{'} + I \otimes I )}$ is $\mathbf{\left( U \otimes U^{'} \right)} \mathbf{\left( D \otimes D^{'} + I \otimes I \right)} \mathbf{\left( U^{T} \otimes U^{'T} \right) }$. The eigenvectors of the noise-added covariance matrix can thus be expressed as Kronecker products, authorizing easy manipulation and fast computation.\\

On the other side, our approach relies on a different implicit factorization. Applying it to the ICM covariance structure yields:

\begin{align*}
\mathbf{\mathcal{K} = HH^{T} \otimes K_{x} + \Sigma \otimes I_{n}} & \simeq \mathbf{\left( U_{H}S_{H} \otimes I_{n}\right) \left( I_{q} \otimes K_{x} \, + \, S_{H}^{-1}U_{H}^{T}\Sigma U_{H}S_{H}^{-1} \otimes I_{n} \right) \left( S_{H} U_{H}^{T} \otimes I_{n} \right) } \\
&= \mathbf{\left( H \otimes I_{n}\right) \left( I_{q} \otimes K_{x} \, + \, H^{+}\Sigma H^{+T} \otimes I_{n} \right) \left( H^{T} \otimes I_{n} \right) }\\ \label{eq:facto_proj}
&\simeq \mathbf{\left( H \otimes I_{n}\right) \left( Diag(K_{x}) \, + \, \Sigma_{P} \otimes I_{n} \right) \left( H^{T} \otimes I_{n} \right) } 
\end{align*}

where $\mathbf{H^{+}}$ denotes the Moore-Penrose pseudoinverse of $\mathbf{H}$. This factorization is not rigorous as it is rank-deficient; it nonetheless appears in the matrix proof of proposition \ref{prop:estimators} in \ref{sec:matrix_der}, where is it made rigorous by two successive applications of Woodburry-like formulas involving other terms. Besides this occurence, this factorization illustrates the core of our approach: the noise-augmented covariance can be summarized by a latent block-diagonal term of size $qn \times qn$, left- and right-multiplied by the mixing matrix.\\

Comparing the above equations shows both the analogy and difference between the two approaches. Ours relies on left- and right-factorization of the the eigen-decomposition of the \emph{task} covariance matrix; the term $\mathbf{H^{+}\Sigma H^{+T}}$ -- in reality $\mathbf{T \Sigma T^{T}}$ with a more rigorous derivation -- has to be diagonal for the central term to be computationally efficient, hence the DPN hypothesis. On the other hand, the approach of \cite{all_in_the_noise} factorizes the eigen-decomposition of the \emph{noise} covariance matrix; then the central factor is always easy to handle because of its $\mathbf{I_{p} \otimes I_{n}}$ term. Note however that \emph{factorization \eqref{eq:ICM_facto} requires the GP covariance matrix to be a Kronecker product}: it is thus restricted to the ICM model.

\end{document}